\newif\ifAMAI \AMAIfalse
\newif\ifklptcompat\klptcompatfalse
\newif\ifGAIacycl\GAIacyclfalse
\ifAMAI
\title{An extended Knowledge Compilation Map for Conditional Preference Statements-based and Generalized Additive Utilities-based Languages}
\else
\title{An extended Knowledge Compilation Map for Conditional Preference Statements-based and Generalized Additive Utilities-based Languages
\\[1em]Research Report IRIT/RR--2023--03--FR 
}
\fi

\date{December 1st, 2023}

\ifAMAI
\RequirePackage{tikz} 
\documentclass{sn-jnl}
\let|\origbar 
\else
\documentclass[10pt]{report}

\usepackage[
 a4paper,
 margin=2cm,head=0ex,headsep=0ex,footskip=2ex    
  ]{geometry}
\usepackage[tt=false, type1=true]{libertine}
       \usepackage[varqu]{zi4}
\def\fnm#1{#1}\def\sur#1{#1}%
\def\email#1{#1}\def\affil#1{#1}
\fi

\usepackage[disable]{todonotes}
\newcommand{\istefan}[1]{\todo[inline,color=yellow]{Stefan: #1}}
\newcommand{\stefan}[1]{\todo[color=yellow]{Stefan: #1}}

\newcommand{\jerome}[1]{\todo[color=yellow]{J: #1}}
\newcommand{\ihel}[1]{\todo[inline,color=lime]{Hel: #1}}
\newcommand{\hel}[1]{\todo[color=lime]{Hel: #1}}
\newcommand{\modifLPTantisym}[1]{#1}
2
\newenvironment{acks}{\paragraph{Acknowledgements}}{\par}
\usepackage{hyperref,multicol,fancyhdr}
\let\citep\cite 

\RequirePackage{ifoption,version,mdframed,xcolor}
\mdfsetup{innerleftmargin=0pt,innerrightmargin=0pt,innerbottommargin=0pt,innertopmargin=0pt}

\mdfdefinestyle{commentaire}{frametitle={COMMENT}}
\mdfdefinestyle{TODO}{frametitle={TO DO}}
\mdfdefinestyle{notAAMAS}{frametitle={not AAMAS:}}
\mdfdefinestyle{AAMAS}{frametitle={AAMAS:}}
\mdfdefinestyle{obsolete}{frametitle={obsolete :}}
\mdfdefinestyle{inappendix}{frametitle={In Appendix :}}

\usepackage[appendix=append,bibliography=common]{apxproof}

\usepackage{amsthm,stmaryrd,mathtools}
	\newtheoremrep{theorem}{Theorem}
	\newtheoremrep{property}[theorem]{Property}
	\newtheoremrep{proposition}[theorem]{Proposition}
	\newtheorem{lemma}{Lemma}

	\newtheorem{definition}{Definition}
	\newtheorem{example}{Example}

\usepackage{amsthm}
\excludeversion{proofsketch}
\excludeversion{notAAMAS}
\excludeversion{commentaire}
\excludeversion{TODO}
\excludeversion{obsolete}
\excludeversion{inappendix}

\usepackage{amsmath}
\usepackage[utf8]{inputenc}
\RequirePackage{eurosym,amsfonts,amssymb,pifont}
\RequirePackage[nointegrals]{wasysym}
\newcommand\MyDeclareUnicodeCharacter[3][]{\DeclareUnicodeCharacter{#2}{#3}}
\MyDeclareUnicodeCharacter[.]{035E}{\bar}
\MyDeclareUnicodeCharacter[.]{035F}{\underline}
\MyDeclareUnicodeCharacter{0394}{\ensuremath{\Delta}}
\MyDeclareUnicodeCharacter{03A3}{\ensuremath{\Sigma}}
\MyDeclareUnicodeCharacter{03B1}{\ensuremath{\alpha}}
\MyDeclareUnicodeCharacter{03B2}{\ensuremath{\beta}}
\MyDeclareUnicodeCharacter{03B3}{\ensuremath{\gamma}}
\MyDeclareUnicodeCharacter{03B4}{\ensuremath{\delta}}
\MyDeclareUnicodeCharacter{03B5}{\ensuremath{\epsilon}}
\MyDeclareUnicodeCharacter{03B8}{\ensuremath{\theta}}
\MyDeclareUnicodeCharacter{03BC}{\ensuremath{\mu}}
\MyDeclareUnicodeCharacter{03C0}{\ensuremath{\pi}}
\MyDeclareUnicodeCharacter{03C3}{\ensuremath{\sigma}}
\MyDeclareUnicodeCharacter{03C4}{\ensuremath{\tau}}
\MyDeclareUnicodeCharacter{03C6}{\ensuremath{\varphi}}
\MyDeclareUnicodeCharacter{03C7}{\ensuremath{\chi}}
\MyDeclareUnicodeCharacter{03C8}{\ensuremath{\psi}}
\MyDeclareUnicodeCharacter{03C9}{\ensuremath{\omega}}
\MyDeclareUnicodeCharacter{03D5}{\ensuremath{\phi}}
\MyDeclareUnicodeCharacter{03F4}{\ensuremath{\Theta}}
\MyDeclareUnicodeCharacter{1D62}{\ensuremath{_i}}
\MyDeclareUnicodeCharacter{2026}{\ensuremath{\ldots}}
\MyDeclareUnicodeCharacter{207F}{\ensuremath{^n}}
\MyDeclareUnicodeCharacter{2080}{\ensuremath{_0}}
\MyDeclareUnicodeCharacter{2081}{\ensuremath{_1}}
\MyDeclareUnicodeCharacter{2082}{\ensuremath{_2}}
\MyDeclareUnicodeCharacter{2083}{\ensuremath{_3}}
\MyDeclareUnicodeCharacter{2084}{\ensuremath{_4}}
\MyDeclareUnicodeCharacter{2085}{\ensuremath{_5}}
\MyDeclareUnicodeCharacter{2086}{\ensuremath{_6}}
\MyDeclareUnicodeCharacter{2087}{\ensuremath{_7}}
\MyDeclareUnicodeCharacter{2088}{\ensuremath{_8}}
\MyDeclareUnicodeCharacter{2089}{\ensuremath{_9}}
\MyDeclareUnicodeCharacter{2096}{\ensuremath{_k}}
\MyDeclareUnicodeCharacter{2098}{\ensuremath{_m}}
\MyDeclareUnicodeCharacter{2099}{\ensuremath{_n}}
\MyDeclareUnicodeCharacter{20AC}{\euro}
\MyDeclareUnicodeCharacter{2115}{\ensuremath{\mathbb{N}}}
\MyDeclareUnicodeCharacter{211A}{\ensuremath{\mathbb{Q}}}
\MyDeclareUnicodeCharacter{211D}{\ensuremath{\mathbb{R}}}
\MyDeclareUnicodeCharacter{2124}{\ensuremath{\mathbb{Z}}}
\MyDeclareUnicodeCharacter{2192}{\ensuremath{\rightarrow}}
\MyDeclareUnicodeCharacter{2193}{\ensuremath{\downarrow}}
\MyDeclareUnicodeCharacter{2194}{\ensuremath{\leftrightarrow}}
\MyDeclareUnicodeCharacter{21D0}{\ensuremath{\Leftarrow}}
\MyDeclareUnicodeCharacter{21D2}{\ensuremath{\Rightarrow}}
\MyDeclareUnicodeCharacter{21D4}{\ensuremath{\Leftrightarrow}}
\MyDeclareUnicodeCharacter{2200}{\ensuremath{\forall}}
\MyDeclareUnicodeCharacter{2203}{\ensuremath{\exists}}
\MyDeclareUnicodeCharacter{2205}{\ensuremath{\emptyset}}
\MyDeclareUnicodeCharacter{2208}{\ensuremath{\in}}
\MyDeclareUnicodeCharacter{2209}{\ensuremath{\notin}}
\MyDeclareUnicodeCharacter{220B}{\ensuremath{\ni}}
\MyDeclareUnicodeCharacter{220F}{\ensuremath{\prod}}
\MyDeclareUnicodeCharacter{2211}{\ensuremath{\sum}}
\MyDeclareUnicodeCharacter{2216}{\ensuremath{\setminus}}
\MyDeclareUnicodeCharacter{2217}{\ensuremath{\ast}}
\MyDeclareUnicodeCharacter{2218}{\ensuremath{\circ}}
\MyDeclareUnicodeCharacter{221A}{\ensuremath{\sqrt}}
\MyDeclareUnicodeCharacter{221D}{\ensuremath{\propto}}
\MyDeclareUnicodeCharacter{221E}{\ensuremath{\infty}}
\MyDeclareUnicodeCharacter{2227}{\ensuremath{\land}}
\MyDeclareUnicodeCharacter{2228}{\ensuremath{\lor}}
\MyDeclareUnicodeCharacter{2229}{\ensuremath{\cap}}
\MyDeclareUnicodeCharacter{222A}{\ensuremath{\cup}}
\MyDeclareUnicodeCharacter{223C}{\ensuremath{\sim}}
\MyDeclareUnicodeCharacter{2248}{\ensuremath{\approx}}
\MyDeclareUnicodeCharacter{2260}{\ensuremath{\neq}}
\MyDeclareUnicodeCharacter{2261}{\ensuremath{\equiv}}
\MyDeclareUnicodeCharacter{2264}{\ensuremath{\leq}}
\MyDeclareUnicodeCharacter{2265}{\ensuremath{\geq}}
\MyDeclareUnicodeCharacter{226A}{\ensuremath{\ll}}
\MyDeclareUnicodeCharacter{226B}{\ensuremath{\gg}}
\MyDeclareUnicodeCharacter{227A}{\ensuremath{\prec}}
\MyDeclareUnicodeCharacter{227B}{\ensuremath{\succ}}
\MyDeclareUnicodeCharacter{227C}{\ensuremath{\preceq}}
\MyDeclareUnicodeCharacter{227D}{\ensuremath{\succeq}}
\MyDeclareUnicodeCharacter{2280}{\ensuremath{\not\prec}}
\MyDeclareUnicodeCharacter{2281}{\ensuremath{\not\succ}}
\MyDeclareUnicodeCharacter{2282}{\ensuremath{\subset}}
\MyDeclareUnicodeCharacter{2283}{\ensuremath{\supset}}
\MyDeclareUnicodeCharacter{2286}{\ensuremath{\subseteq}}
\MyDeclareUnicodeCharacter{2287}{\ensuremath{\supseteq}}
\MyDeclareUnicodeCharacter{2288}{\ensuremath{\not\subseteq}}
\MyDeclareUnicodeCharacter{2289}{\ensuremath{\not\supseteq}}
\MyDeclareUnicodeCharacter{228F}{\ensuremath{\sqsubset}}
\MyDeclareUnicodeCharacter{2290}{\ensuremath{\sqsupset}}
\MyDeclareUnicodeCharacter{2291}{\ensuremath{\sqsubseteq}}
\MyDeclareUnicodeCharacter{2292}{\ensuremath{\sqsupseteq}}
\MyDeclareUnicodeCharacter{22A2}{\ensuremath{\vdash}}
\MyDeclareUnicodeCharacter{22A4}{\ensuremath{\top}}
\MyDeclareUnicodeCharacter{22A5}{\ensuremath{\bot}}
\MyDeclareUnicodeCharacter{22A7}{\ensuremath{\models}}
\MyDeclareUnicodeCharacter{22A8}{\ensuremath{\models}}
\MyDeclareUnicodeCharacter{22AD}{\ensuremath{\not\models}}
\MyDeclareUnicodeCharacter{22B3}{\ensuremath{\vartriangleright}}
\MyDeclareUnicodeCharacter{22C0}{\ensuremath{\bigwedge}}
\MyDeclareUnicodeCharacter{22C1}{\ensuremath{\bigvee}}
\MyDeclareUnicodeCharacter{22C2}{\ensuremath{\bigcap}}
\MyDeclareUnicodeCharacter{22C3}{\ensuremath{\bigcup}}
\MyDeclareUnicodeCharacter{22C5}{\ensuremath{\boldsymbol{\cdot}}}
\MyDeclareUnicodeCharacter{22C8}{\ensuremath{\bowtie}}
\MyDeclareUnicodeCharacter{22D0}{\ensuremath{\Subset}}
\MyDeclareUnicodeCharacter{22D1}{\ensuremath{\Supset}}
\MyDeclareUnicodeCharacter{22E0}{\ensuremath{\not\preceq}}
\MyDeclareUnicodeCharacter{22E1}{\ensuremath{\not\succeq}}
\MyDeclareUnicodeCharacter{22EB}{\ensuremath{\ntriangleright}}
\MyDeclareUnicodeCharacter{22EE}{\ensuremath{\vdots}}
\MyDeclareUnicodeCharacter{22EF}{\ensuremath{\cdots}}
\MyDeclareUnicodeCharacter{223C}{\ensuremath{\sim}}
\MyDeclareUnicodeCharacter{25A1}{\ensuremath{\square}}
\MyDeclareUnicodeCharacter{25EF}{\ensuremath{\bigcirc}}
\MyDeclareUnicodeCharacter{2713}{\textrm{\ding{51}}}
\MyDeclareUnicodeCharacter{2714}{\textrm{\ding{51}}}
\MyDeclareUnicodeCharacter{2718}{\textrm{\ding{56}}}
\MyDeclareUnicodeCharacter{27E6}{\llbracket}
\MyDeclareUnicodeCharacter{27E7}{\rrbracket}
\MyDeclareUnicodeCharacter{27F5}{\ensuremath{\longleftarrow}}
\MyDeclareUnicodeCharacter{27F6}{\ensuremath{\longrightarrow}}
\MyDeclareUnicodeCharacter{27F7}{\ensuremath{\longleftrightarrow}}
\MyDeclareUnicodeCharacter{27FC}{\ensuremath{\mapsto}}

\usepackage{rotating}

\newcommand\mathnamestyle{\sffamily}
\DeclareUnicodeCharacter{00A7}{\mathnamedelim}
{\catcode`\_=13%
\gdef\makesoulignesouligne{\catcode`\_=13\def_{\_}}%
\gdef\makesoulignesubscript{\catcode`\_=8}}
\let\jmtempa\relax
\protected\def\mathnamedelim{\ifx\jmtempa\relax\mathnameinit
	\else\mathnameend\fi}
\def\mathnameinit{\ifmmode\mbox\bgroup\else\bgroup\fi
		\makesoulignesouligne\let\jmtempa\mathnamestyle\jmtempa}
\def\mathnameend{\egroup\let\jmtempa\relax
	\makesoulignesubscript}

\usepackage{float}
\floatstyle{plaintop}
\newfloat{algorithmflt}{t}{algos}\floatname{algorithmflt}{Algorithm}
\newcommand{\myalgocaption}{}

\newenvironment{algo}[2][Algorithm]
		{\par\vskip 1ex plus.2ex minus.2ex\noindent{\bf #1\ : #2}\enumerate}
		{\endenumerate}
		
\usepackage{enumitem}
\newlist{langlist}{description}{1}
\setlist[langlist]{labelindent=0pt,leftmargin=1.5em,labelsep=1ex,align=left}

\newcommand{\acycl}{\ensuremath{\mathord{\not \mathrel{\circlearrowright}}}}

\newcommand{\acycllex}{\ensuremath{\mathord{\not \mathrel{\circlearrowright}^{§lex§ }}}}
\newcommand{\acycllexk}{\ensuremath{\mathord{\not \mathrel{\circlearrowright}_k^{§lex§ }}}}
\newcommand{\acycllexkun}{\ensuremath{\mathord{\not \mathrel{\circlearrowright}_{k-1}^{§lex§ }}}}
\newcommand{\acycllexun}{\ensuremath{\mathord{\not \mathrel{\circlearrowright}_1^{§lex§ }}}}
\newcommand{\acyclpoly}{\ensuremath{\mathord{\not \mathrel{\circlearrowright}^{§poly§}}}}
\let\cal\mathcal
\newcommand\card[1]{\mathord{\vert #1\vert}}
\newcommand{\cclass}[1]{\ensuremath{\mbox{\textsf{#1}}}}

\newcommand\coNP{\ensuremath{\cclass{coNP}}}
\newcommand{\defquery}\relax \def \defquery#1#2!{\trivlist \item \pb{#1} \ #2 \endtrivlist}
\newcommand\dom{\underline}
\newcommand\domX{\dom{\cal X}}

\newcommand{\mexpr}{\mathrel{\sqsupset}}
\newcommand{\mexpreq}{\mathrel{\sqsupseteq}}
\newcommand{\mexprm}{\mathrel{\sqsubseteq\!\!\!\sqsupseteq}}
\newcommand{\msucc}{\ll}
\newcommand{\msucceq}{\leqq}
\newcommand\N{\ensuremath{\mathbf N}}
\makeatletter
\newcommand\nouv[1]{#1}
\makeatother
\newcommand\NP{\ensuremath{\cclass{NP}}}
\newcommand\NPc{\ensuremath{\cclass{NPc}}}
\newcommand\NPh{\ensuremath{\cclass{NPh}}}
\newcommand{\NPSPACE}{\ensuremath{\cclass{NPSPACE}}}
\newcommand\Out{\tikz [x=1em,y=1em,line width=.1ex] \draw (0,0) -- (1,1) (0,1) -- (1,0);}
\newcommand{\sP}{\ensuremath{\cclass{\#P}}}
\renewcommand\P{\ensuremath{\cclass P}}
\renewcommand\phi{\varphi}
\newcommand{\pb}{\textsc}
\newcommand{\PSPACE}{\ensuremath{\cclass{PSPACE}}}
\newcommand{\PSPc}{\ensuremath{\cclass{PSPc}}}

\newcommand\vh[1]{\multicolumn1c{\rotatebox{\vhangle}{{\parbox{\vhheight}{#1}}}}} \newcommand\vhangle{90}
	\newlength{\vhheight}
\usepackage{tikz}
	\usetikzlibrary{shapes.geometric,arrows,arrows.meta,calc,positioning,fit}
\tikzset{graphe/.style={x=4em,y=-4em,node distance=4em}}
\tikzstyle{pref}=[line width=1pt,>=latex,->]
\tikzstyle{pref}=[line width=1pt,>=latex,->]
\tikzset{var/.style={rectangle,rounded corners,draw,inner sep=1pt,outer sep=0pt,minimum size=1.4em}}

\DeclareMathSymbol{:}{\mathop}{operators}{"3A}
\DeclareMathDelimiter{|}{\mathop}{symbols}{"6A}{largesymbols}{"0C}
\DeclareMathSymbol{>}{\mathop}{letters}{"3E}
\DeclareMathSymbol{\leq}{\mathop}{symbols}{"14}
\DeclareMathSymbol{\geq}{\mathop}{symbols}{"15}
\DeclareMathSymbol{\ntriangleright} {\mathord}{AMSb}{"37}

\ifAMAI\else
  
  \renewcommand{\abstract}[1]{\begin{abstractenv}#1\end{abstractenv}}
  \newcommand\keywords[1]{}
\fi

\begin{document}

\ifAMAI

\abstract{Conditional preference statements have been used to compactly represent preferences over combinatorial domains. They are at the core of CP-nets and their generalizations, and lexicographic preference trees.  Several works have addressed the complexity of some queries (optimization, dominance in particular). We extend in this paper some of these results, and study other queries which have not been addressed so far, like equivalence, and transformations, like conditioning and variable elimination, thereby contributing to a knowledge compilation map for languages based on conditional preference statements. We also study the expressiveness and complexity of queries and transformations for generalized additive utilities.
\stefan{il faut aussi parler des GAI déjà dans l'abstract, non ?}}


\author*[1]{\fnm{Hélène} \sur{Fargier}}\email{Helene.Fargier@irit.fr}

\author*[2]{\fnm{Stefan} \sur{Mengel}}\email{mengel@cril-lab.fr}

\author*[1]{\fnm{Jérôme} \sur{Mengin}}\email{Jerome.Mengin@irit.fr}

\affil[1]{IRIT, Université Paul Sabatier, CNRS, 118 route de Narbonne, 31062 Toulouse Cedex 9, France}

\affil[2]{Université d’Artois, CNRS, Centre de Recherche en Informatique de Lens (CRIL), Lens, France}

\maketitle

\else

\author{Hélène Fargier,\\\url{Helene.Fargier@irit.fr},
  \\IRIT, Université Paul Sabatier, CNRS, 118 route de Narbonne, 31062 Toulouse Cedex 9, France
\and Stefan Mengel,\\\url{mengel@cril-lab.fr},
  \\Université d’Artois, CNRS, Centre de Recherche en Informatique de Lens (CRIL), Lens, France
\and Jérôme Mengin,\\\url{Jerome.Mengin@irit.fr},
  \\IRIT, Université Paul Sabatier, CNRS, 118 route de Narbonne, 31062 Toulouse Cedex 9, France
}

\maketitle

\abstract{Conditional preference statements have been used to compactly represent preferences over combinatorial domains. They are at the core of CP-nets and their generalizations, and lexicographic preference trees.  Several works have addressed the complexity of some queries (optimization, dominance in particular). We extend in this paper some of these results, and study other queries which have not been addressed so far, like equivalence, and transformations, like conditioning and variable elimination, thereby contributing to a knowledge compilation map for languages based on conditional preference statements. We also study the expressiveness and complexity of queries and transformations for generalized additive utilities, and introduce a new parameterized family of languages, which enables to balance expressiveness against the complexity of some queries.\stefan{il faut aussi parler des GAI déjà dans l'abstract, non ?} This paper is an extended version of \citep{FargierMengin:AAMAS21} -- in addition to the results of \citep{FargierMengin:AAMAS21},  it contains a study of several  transformations (Section~\ref{sec:transfo}). We have also added the GAI language to the map.
}

\fi


\begin{TODO}
\begin{itemize}
\item Check complexity results by \cite{LukasiewiczMalizia:art-int19}
\item \pb{equivalence}: is it hard for CP 1 conj no free vars k-lexico comp ?
\item Lien avec \textit{CI-statements} de \cite{BouveretEndrissLang:ijcai09}
\item \nouv{Communication de Stefan : $>$-cut-counting is NP-hard for acyclic CP-nets.}
\item Langages supplémentaires :
  \begin{itemize}
  \item §GAI dec.§\ = Generalized additive utility;
  \item $ §GAI dec.§_k = §GAU§ + $ borne $ k $ sur la taille des cliques
  \item §GAI§-net ?? (mais c'est quoi exactement ?)
  \end{itemize}
\end{itemize}
\end{TODO}

\begin{TODO}
\begin{itemize}
\item expressivité moyennant introduction de nouvelles attributes ?
\item langages logiques, cf \cite{CosteMarquisLangLiberatoreMarquis:kr04}
\item succinctness entre 2 langages pas aussi expressifs
\item approx. sup. et/ou inf. d'un langage par un autre
\item pour les requêtes, $ ≻ $ ou $ ≽ $ ?
	\begin{itemize}
	\item pour dominance : $ ≽ $, à partir de laquelle on recalcule les autres pour GAI, et qui est équivalente à $ ≻ $ pour les CP-statements;
	\item pour MX, ME: peut-on aussir calculer plus efficacement $ §CUT§^{≻} $ sans passer par $ §CUT§^{≽} $ ?
	\end{itemize}
\end{itemize}
\end{TODO}

\ifAMAI
\paragraph{Data availability statement} Data sharing not applicable to this article as no datasets were generated or analysed during the current study.
\fi

\section{Introduction}

Preference handling is a key component in several areas of Artificial Intelligence, notably for decision-aid systems. Research in Artificial Intelligence has led to the development of several languages that enable compact representation of preferences over complex, combinatorial domains. Some preference models rank alternatives according to their values given by some multivariate function; this is the case for instance with valued constraints \citep{SchiexFargierVerfaillie:ijcai95}, additive utilities and their generalizations \citep{GonzalesPerny:kr04,BraziunasBoutilier:uai05}. Ordinal models like CP nets and their generalizations~\citep{Boutilieretal:jair04,Wilson:aaai04,Brafmanetal:jair06}, or lexicographic preferences and their generalizations \citep{GigerenzerG:psych-reviewG96,SchmittMartignon:jmlr06,Wilson:ecai06,Boothetal:ecai10,BrauningHullermeier:pl12,FargierGimenezMengin:aaai18} use sets of conditional preference statements to represent a pre-order over the set of alternatives.

Many problems of interest, like comparing alternatives or finding optimal alternatives,  are NP-hard for many of these models, and in fact even PSPACE-hard for some of them, which makes these representations difficult to use in some decision-aid systems like configurators, where real-time interaction with a decision maker is needed. One approach to tackling this problem is Knowledge Compilation, which is a general approach in which a model, or a part of it, is \textit{compiled}, off-line, into another representation which enables fast query answering, even if the compiled representation has a much bigger size. This approach has first been studied in propositional logic: \cite{Darwiche:ijcai99,DarwicheMarquis:jair02} compare how various subsets of propositional logic can succinctly, or not, express propositional knowledge bases, and the complexity of queries of interest. \cite{CosteMarquisLangLiberatoreMarquis:kr04} follow a similar approach to compare extensions of propositional logic which associate real values to models of a knowledge base; \cite{FargierMarquisNiveauSchmidt:aaai14} consider value function-based models.

The aim of this paper is to initiate a \emph{compilation map} for representations on preferences. To this end, we systematically study and compare different languages of conditional preference statements and models based on Generalized Additive Utilities (called GAIs). In particular, we analyze the expressiveness and succinctness of various languages based on these conditional preference statements and on GAIs, and the complexity of several queries and transformations of interest.

Section~\ref{sec:prelim} recalls some basic definitions about combinatorial domains and pre-orders, and introduces notation that we will use throughout.  Section~\ref{sec:language} gives an overview of various languages based on conditional preference statements that have been studied in the literature. It introduces first a general language of conditional preference statements, and recalls the language of Generalized Additive Utilities. The remainder of this section then presents various language restrictions that have been studied in the literature and offer interesting compromises between expressiveness and querying complexity.


Section~\ref{sec:expressiveness} and~\ref{sec:succinctness} respectively study expressiveness and succinctness for the languages we study. Sections~\ref{sec:queries} and~\ref{sec:transfo} study the complexity of, respectively,
queries and transformations for these languages.

This paper is an extended version of \citep{FargierMengin:AAMAS21} -- in addition to the results of \citep{FargierMengin:AAMAS21},  it contains a study of several  transformations (Section~\ref{sec:transfo}). We have also added the GAI language to the map. (Unpublished) proofs are provided in the appendix.

\section{Preliminaries}\label{sec:prelim}

\subsection{Combinatorial Domains}

We consider languages that can be used to represent the preferences of a decision maker over a combinatorial space $ \dom {\cal X } $: here $ \cal X   $ is a set of attributes that characterize the possible alternatives, each attribute $ X ∈ \cal X  $ having a finite set of possible values $ \dom X $; we assume $ | \dom X | ≥ 2 $ for every $ X ∈ \cal X $; then $ \dom{\cal X } $ denotes the Cartesian product of the domains of the attributes in $ \cal X  $, its elements are called alternatives. For a binary attribute $ X $, we will often denote by $ x, \bar  x $ its two possible values. In the sequel, $ n $ is the number of attributes in $ \cal X $.

For a subset $ U $ of $ \cal X  $, we will denote by $ \dom U $ the Cartesian product of the domains of the attributes in $ U $. The elements of $\dom U$  are called called  instantiations of $ U $, or partial instantiations (of $ \cal X  $). If $ v $ is an instantiation of some $ V ⊆ \cal X  $, $ v[U] $ denotes the restriction of $ v $ to the attributes in $ V \cap U $; we say that instantiation $ u ∈ \dom U $ and $ v $ are compatible if $ v[U∩V] = u[U∩V] $; if $ U ⊆ V $ and $ v[U] = u $, we say that $ v $ extends $ u $.

Sets of partial instantiations can often be conveniently, and compactly, specified with propositional formulas: the atoms are $ X=x $ for every $ X ∈ \cal X  $ and $ x ∈ \dom X $, and we use the standard connectives $ ∧ $ (conjunction), $ ∨ $ (disjunction), $ → $ (implication), $ ↔ $ (equivalence) and $ ¬ $ (negation); we denote by $ ⊤ $ (resp. $ ⊥ $) the formula always true (resp. false). Implicitly, this propositional logic is equipped with a theory that enforces that every attribute has precisely one value from its domain; so, for two distinct values $ x, x' $ of attribute $ X $, the formula $ X=x ∧ X= x' $ is a contradiction; also, the interpretations are thus in one-to-one correspondence with $ \dom{\cal X } $. If $ α $ is such a propositional formula over $ \cal X  $ and $ o ∈ \dom{\cal X } $, we will write $ o ⊧ α $ when $ o $ satisfies $ α $, that is when, assigning to every literal $ X = x $ that appears in $ α $ the value $ §true§ $ if $ o[X] = x $, and the value $ §false§ $ otherwise, makes $ α $ true.

Given a formula $ α $, or a partial instantiation $ u $,
$ §Var§(α) $ and $ §Var§(u) $ denote the set of attributes, the values of which appear in $ α $ and $ u $ respectively.

When it is not ambiguous, we will use $ x $ as a shorthand for the literal $ X=x $; also, for a conjunction of such literals, we will omit the $ ∧ $ symbol, thus $ X=x ∧ Y=\bar y $ for instance will be denoted $ x\bar y $.
\begin{notAAMAS}
Conversely, given a partial instantiation $ u ∈ ͟U $ for some $ U ⊆ \cal X $, $ u $ can also denote the formula $ ⋀_{X ∈ U} X=u[X] $; in particular, two partial instantiations $ u $ and $ v $ are compatible if and only if the formula $ u ∧ v $ is consistent, that is, if and only if $ u ∧ v ⊭ ⊥ $.
\end{notAAMAS}

\begin{commentaire}\mdfsubtitle{Si attributs binaires :}
For ease of presentation, we consider binary attributes only: for attribute $ X $, we denote by $ x $ and $ \bar x $ its two possible values. Sets of partial instantiations can often be conveniently, and compactly, specified with propositional formulas: the propositions are $ X=x $ for every $ X ∈ \cal X  $, 
and we use the standard connectives $ ∧ $ (conjunction), $ ∨ $ (disjunction), $ → $ (implication), $ ↔ $ (equivalence) and $ ¬ $ (negation). When it is not ambiguous, we will use $ x $ (respectively $ \bar x $) as a shorthand for the literal $ X=x $ (respectively for the literal $ ¬(X=x) $ that constrains $ X $ to take value $ \bar x $). When used in the context of a formula, a partial instantiation $ u ∈ \dom U $ is used as a shorthand for the formula $ ⋀_{X∈U} X=u[X] $. Given a formula $ α $, or a partial instantiation $ u $, $ §Var§(α) $ and $ §Var§(u) $ denote the set of attributes, the values of which appear in $ α $, or $ u $, respectively. 
\end{commentaire}

\subsection{Preference Relations}


Depending on the knowledge that we have about a decision maker's preferences, given any pair of distinct alternatives $ o,  o' ∈  \dom{\cal X } $, one of the following situations must hold: one may be strictly preferred over the other, or $ o $ and $ o' $ may be equally preferred, or $ o $ and $ o' $ may be incomparable.

Assuming that preferences are transitive, such a state of knowledge about the decision maker's preferences can be characterized by a preorder $ ≽ $ over $ \dom{\cal X } $, that is $ ≽ $ is a binary, reflexive and transitive relation. For alternatives $ o, o' $, we write
{\let\item\relax
\item $ o ≽ o' $ when $ (o,o') ∈ \mathord{≽} $;
\item $ o ≻ o' $ when $ (o,o') ∈ \mathord{≽} $ and $ (o',o) ∉ \mathord{≽} $;
\item $ o ∼ o' $ when $ (o,o') ∈ \mathord{≽} $ and $ (o',o) ∈ \mathord{≽} $;
\item $ o ⋈ o' $ when $ (o,o') ∉ \mathord{≽} $ and $ (o',o) ∉ \mathord{≽} $.
}
Note that for any pair of alternatives $ o, o' ∈ \dom{\cal X } $ either $ o ≻ o' $, or $ o' ≻ o $, or $ o ∼ o' $ or $ o ⋈ o' $ holds.

The relation $ ∼ $ defined in this way is called the \emph{symmetric part} of $ ≽ $; it is symmetric, reflexive and transitive. The relation $ ⋈ $ is symmetric and irreflexive. The relation $ ≻ $ is called the \emph{asymmetric part} of $ ≽ $, and is what is usually called a strict partial order, i.e., it is irreflexive, transitive and asymmetric.

When the preorder $ ≽ $ is complete, that is, when it is the case that $ o ≽ o' $ or $ o' ≽ o $ for every pair of alternatives $ (o,o') $, it is called a \emph{weak order}. A strict partial order that is complete is called a \emph{linear order}.

When the preorder $ ≽ $ is antisymmetric, that is when $ o ∼ o' $ only when $ o=o' $, then it is called a \emph{partial order}.

\paragraph{Terminology and notation}
We say that an alternative $ o $ \emph{dominates} an alternative $ o ' $ (w.r.t.~$ ≽ $) if and only if $ o ≽ o' $. If $ o ≻ o' $, then we say that $ o $ \emph{strictly dominates} $ o' $.
We use standard notation for the complements of $ ≻ $ and $ ≽ $: we write $ o ⋡ o' $ when it is not the case that $ o ≽ o' $,  and $ o ⊁ o' $ when it is not the case that $ o ≻ o'$. Given two preorders $ ≽ $ and $ ≽' $, we say that $ ≽ $ \emph{extends} $ ≽' $ when $ o ≽' o $ implies $ o ≽ o'$, for every pair of alternatives $ o,o' $.

%

\section{Languages}
\label{sec:language}

\subsection{Conditional Preference Statements}

A \emph{conditional preference statement} (short \emph{CP statement}) over $ \cal X  $ is an expression of the form $ α | V : w ≥ w' $, where $ α $ is a propositional formula over $ U ⊆ \cal X  $, $ w, w' ∈ \dom W $ are such that $ w[X] ≠ w'[X] $ for every $ X ∈ W $, and $ U, V, W $ are disjoint subsets of $ \cal X  $, not necessarily forming a partition of $ \cal X $.  Informally, such a statement represents the piece of knowledge that, when comparing alternatives $ o, o' $ that both satisfy $ α $, the one that has values $ w $ for $ W $ is preferred to the one that has values $ w' $ for $ W $, irrespective of the values of the attributes in $ V $, every attribute in $ \cal X ∖ (V ∪ W) $ being fixed. We call $ α $ the conditioning part of the statement; we call $ W $ the swapped attributes, and $ V $ the free part.

\begin{example}[({Example A in \cite{Wilson:aij11},} slightly extended)]\label{exple:CP-th}
Consider planning a holiday, with three choices / attributes: wait until next month ($W=w$) or leave now ($W=\bar w$), going to city 1, 2 or 3 ($ C=c₁ $, $ C=c₂ $ or $ C=c₃ $), travelling by plane ($ P=p $) or by car ($P=\bar p$). I would rather go now, irrespective of the other attributes: $ ⊤|\{CP\}:\bar w≥w $. All else being equal, I prefer to go to city 3, city 1 being my second best choice: $ ⊤|∅:c₃ ≥ c₁ ≥ c₂ $. Also, if I go now, I prefer to fly: $ \bar w|∅:p ≥ \bar p $. Together, the last two statements imply that if I go now, I prefer to go to city 3 by plane than go to city 1 by car; however these statements do not say what I prefer between flying to city 1 or driving to city 3. In fact, I prefer the former, this \textit{tradeoff} can be expressed with the statement $ \bar w|∅:c₁p ≥ c₃\bar p $. Finally, if I go later, I prefer to drive, irrespective of the city: $ w|\{C\}:\bar p ≥p $.
\end{example}

Conditional preference statements have been studied in many works, under various language restrictions.
They are the basis for CP-nets \citep{BoutilierBrafmanHoosPoole:uai99,Boutilieretal:jair04} and their extensions, and have been studied in a more logic-based fashion by e.g. \cite{Goldsmithetal:jair08,Wilson:aaai04,Wilson:ecai04,Wilson:aij11}.
\footnote{The formula $ u | V : x≥x' $ is written $ u:x>x' [V] $ by \cite{Wilson:aij11}.}
Closely related to them are the \textit{Conditional Importance statements} studied in \cite{BouveretEndrissLang:ijcai09}.

For the semantics of sets of CP statements, we use the definitions of \cite{Wilson:aij11}. Given a statement $ α|V:w≥w' $, let $ U = §Var§(α) $ and $ W = §Var§(w) = §Var§(w') $: a \emph{worsening swap} is any pair of alternatives $ (o, o') $ such that $ o[U] = o'[U] ⊧ α $, $ o[W] = w $ and $ o'[W] = w' $, and such that for every attribute $ Y ∉ U ∪ V ∪ W  $ it holds that $ o[Y] = o'[Y] $; we say that $ α|V:w≥w' $ \emph{sanctions} $ (o,o') $.
%
%
For a set of CP-statements $ φ $, let $ φ^* $ be the set of all worsening swaps sanctioned by statements of φ, and define $ ≽_φ $ to be the reflexive and transitive closure of $ φ^* $.  \cite{Wilson:aij11} proves that $ o ≽_φ o' $ holds if and only if $ o = o' $ holds or $ φ^* $ contains a finite sequence of worsening swaps $ (oᵢ, o_{i+1})_{0≤i≤k-1} $ with $ o₀ = o $ and $ oₖ = o' $.\footnote{ Actually, \cite{Wilson:aij11} proves that $ (o,o') $ is in the transitive closure of $ φ^* $ if and only there is such a worsening sequence from $ o $ to $ o' $, but adding the reflexive closure to this transitive closure does not change the result, since we can add any pair $ (o,o) $ to, or remove it from, any sequence of worsening swaps without changing the validity of the sequence.}

\begin{example}[Example~\ref{exple:CP-th}, continued] \label{exple:CP-thbis}
Let
$$ φ = \left\{ \begin{array}{c}
  ⊤|\{CP\}:\bar w≥w \ , \ ⊤|∅:c₃ ≥ c₁ ≥ c₂ , \\
  \bar{w}|∅:p ≥ \bar p \ , \ \bar w|∅:c₁p ≥ c₃\bar p \ ,\ w|\{C\}:\bar p ≥p \} 
\end{array} \right\}. $$
Then   $⊤|\{CP\}:\bar w≥w $ sanctions for instance $ (\bar wc₂p, wc₃\bar p) $, so $ \bar wc₂p ≽_φ wc₃\bar p $. Also, $ ⊤|∅:c₃ ≥ c₁ ≥ c₂ $ sanctions $ (\bar wc₁p,\bar wc₂p) $, $ \bar w|∅:p ≥ \bar p$ sanctions $ (\bar wc₂p,\bar wc₂\bar p) $, so, by transitivity, $ \bar wc₁p ≽_φ \bar wc₂\bar p $. It is not difficult to check that $ \bar wc₂p ⋈_φ\bar wc₁\bar p $.
\end{example}

Let us call §CP§\ the language where formulas are sets of statements of the general form $ α | V : w ≥ w' $. This language is very expressive: it is possible to represent any preorder ``in extension'' with preference statements of the form $ o ≥ o' $ -- they have $ W = \{ X | o[X] ≠ o'[X] \} $ as set of swapped attributes, $ α = o[U] = o'[U] $ as condition where $ U = \{ X | o[X] = o'[X] \}  $, and no free attribute. 

This expressiveness has a cost: we will see that many queries about pre-orders represented by §CP§-statements are \PSPACE-hard for the language §CP§. Several restrictions / sublanguages have been studied in the literature, we review them below.


\paragraph{(Strict) Consistency} Although the original definition of CP-nets by \cite{BoutilierBrafmanHoosPoole:uai99} does not impose it, many works on CP-nets, especially following \cite{Boutilieretal:jair04}, consider that they are intended to represent a strict partial order, that is, that $ ≽_φ $ should be antisymmetric.
We say that a set $ φ $ of CP-statements is \emph{consistent} in this case. Note that in this case, for two different alternatives $ o $ and $o' $, $ o ≽_φ o'$ implies that $ o ≻_φ o' $.

%

\paragraph{Notation}

\begin{obsolete}
For statement $ s = α | V : w ≥ w' $, we denote by $ s^d $ the statement where $ w $ and $ w' $ are exchanged, we call it the \emph{dual} of $ s $: $ s^d =  α | V : w' ≥ w $; and, if $ φ $ is a set of CP statements, $ φ^d $ will denote the set of the duals of the statements of $ φ $, we call it the dual of $ φ $. It can be checked that the relation $ ≽_{φ^d} $ is the dual of $ ≽_φ $, that is, $ ≽_{φ^d} = ≼_φ $; and similarly, $ ≻_{φ^d} = ≺_φ $. Note that $ φ^d $ can be computed in linear time from $ φ $.
\end{obsolete}

We write $ α : w ≥ w' $ when $ V $ is empty, and $ w ≥ w' $ when $ V $ is empty and $ α = ⊤ $. Note that we reserve the symbol $ ≥ $ for conditional preference statements, whereas ``curly'' symbols $ ≻ $, $ ⊁ $, $ ≽ $, $ ⋡ $ are used to represent relations over the set of alternatives.



In the remainder of this section, we present various sublanguages of §CP§. Some are defined by imposing various simple syntactical restrictions on the formulas, two are languages which have been well studied (CP-nets and lexicographic preference trees).

\subsection{Statement-wise Restrictions}

Some restrictions are on the syntactical form of statements allowed; they bear on the size of the set of free attributes, or on the size of the set of swapped attributes, or on the type of conditioning formulas allowed. Given some language $ \cal L  ⊆ §CP§ $, we define the following restrictions:
\begin{langlist}

\item[$ \cal L ⋫ $ =] only formulas with empty free parts ($ V  = ∅ $) for every statement;\footnote{In the literature, the symbol $ ⊳ $ is sometimes used to represent an \textit{importance} relation between attributes; and, as explained by \cite{Wilson:aij11}, statement $ α | V : w ≥ w' $ is a way to express that attributes in $ §Var§(w) $ are more important than those in $ V $ (when $ α $ is true).}
\item[$ \cal L ∧ $ =] only formulas where the condition $ α $ of every statement is a conjunction of literals;


\item[$ §k-§\cal L $ =] only formulas where the set of swapped attributes contains no more than $ k $ attributes ($ |W| ≤ k $) for every statement; in particular, we call elements of §1-CP§\ \textit{unary} statements.

\end{langlist}
In particular, $ §1-CP∧§ $ is the language studied by \cite{Wilson:aij11}, and $ §1-CP⋫§ $ is the language of generalized CP-nets as defined by \cite{Goldsmithetal:jair08}.

\subsection{Graphical Restrictions} \label{sec:graph-restr}

Given $ φ ∈ §CP§ $ over set of attributes $ \cal X  $, we define $ D_φ $ as the graph with sets of vertices $ \cal X  $, and such that there is an edge $ (X,Y) $ if there is $ α | V : w≥w' ∈ φ $ such that $ X ∈ §Var§(α) $ and $ Y ∈ §Var§(w) $, or $ X ∈ §Var§(w) $ and $ Y ∈ V $. We call $ D_φ $ the \emph{dependency graph} of $ φ $. Note that $ D_φ $ can be computed in polynomial time. This definition,  inspired by \citep[Def. 15]{Wilson:aij11}, generalizes that of \cite{Boutilieretal:jair04}, which is restricted to the case where all CP statements are unary and have no free attributes, and that of  \cite{Brafmanetal:jair06}, who study statements with free attributes. 
Many tractability results on sets of CP statements have been obtained when $ D_φ $ has good properties. Given some language $ \cal L  ⊆ §CP§ $, we define:
\begin{langlist}

\item[$\cal L  \acycl $ =] the restriction of $ \cal L  $ to \textit{acyclic} formulas, which are those $ φ $ such that $ D_φ $ is acyclic;\footnote{This is \textit{full acyclicity} in \citep{Wilson:aij11}.}

\item[$\cal L  \acyclpoly $ =] the restriction of $ \cal L  $ to formulas where the dependency graph is a polytree.


\end{langlist}

\ifklptcompat
\cite{Wilson:aij11} also defines a weaker graphical restriction, called  ``context-uniform conditional acyclicity'', but it turns out that it does gives rise to the same complexities as another, weaker restriction called ``conditional acyclicity'' by \cite{Wilson:aij11}, which we generalize in section~\ref{sec:lex-comp}.
\fi



\begin{TODO}
Check what is the weakest condition that guarantees that dominance can be checked in polynomial time for CP nets.
\end{TODO}

\subsection{§CP§-nets}

In their seminal work,  \cite{Boutilieretal:jair04} define a CP-net over a set of attributes $ \cal X  $ to be composed of two elements:
\begin{enumerate}

\item a directed graph over $ \cal X  $, which should represent \textit{preferential dependencies} between attributes;\footnote{Given some pre-order $ ≽ $ over $ \cal X  $, attribute $ X $ is said to be preferentially dependent on attribute $ Y $ if there exist $ x, x' ∈ \dom X $, $ y, y' ∈ \dom Y $, $ z  ∈ \dom {\cal X  ∖ (\{X,Y\})} $ such that $ xyz ≽_φ x'yz $ but  $ xy'z ⋡_φ x'y'z $.}

\item a set of conditional preference tables, one for every attribute $ X $: if $ U $ is the set of parents of $ X $ in the graph, the conditional preference table for $ X $ contains exactly $ |\dom U| $ rules $ u:≥ $, for every $ u ∈ \dom U $, where the $ ≥ $'s are linear orders over $ \dom X $.

\end{enumerate}
Therefore, as shown by \cite{Wilson:aij11}, CP-nets can be seen as sets of unary §CP§\ statements in conjunctive form with no free attribute. Specifically, given a CP-net $ \cal N  $ over $ \cal X  $, define $ φ_{\cal N}  $ to be the set of all CP statements $ u:x ≥ x' $, for every attribute $ X $, every $ u \in \dom U $ where $ U $ is the set of parents of $ X $ in the graph, every $ x, x' \in \dom X $ such that $ x,x' $ are consecutive values in the linear order $ ≥ $ specified by the rule $ u:≥ $ of $ \cal N  $. Then the dependency graph of $ φ_{\cal N} $, as defined in Section~\ref{sec:graph-restr}, coincides with the graph of $ \cal N $. We call

\begin{langlist}
\item[$ §CPnet§ $ =] the language that contains all $ φ_{\cal N}  $, for every CP-net $ \cal N  $.
\end{langlist}

Note that $ §CPnet§ ⊆ §1-CP§\mathord{∧}⋫ $. For a given $ φ ∈ §1-CP§\mathord{∧}⋫ $, being a CP-net necessitates a very strong form of local consistency and completeness: 
for every attribute $ X $ with parents $ U $ in $ D_φ $, for every $ u ∈ \dom{U} $, for every $ x, x' ∈ \dom X $, φ must explicitly, and uniquely, order $ ux $ and $ ux' $.


\cite{Brafmanetal:jair06} define TCP-nets as an extension of CP-nets where it is possible to represent tradeoffs, by stating that, under some conditions, some attributes are more important than other ones. \cite{Wilson:aij11} describes how TCP-nets can be transformed, in polynomial time, into equivalent sets of §1-CP∧§\ statements.

\subsection{Lexicographic Preference Trees}

LP-trees generalize lexicographic orders, which have been widely studied in decision making -- see e.g.~\cite{Fishburn:managsc74}. As an inference mechanism, they are equivalent to search trees used by \cite{Boutilieretal:compint04}, and formalized by \cite{Wilson:ecai04,Wilson:aij11}. As a preference representation, and elicitation, language, slightly different definitions for LP-trees have been proposed by \cite{Boothetal:ecai10,BrauningHullermeier:pl12,FargierGimenezMengin:aaai18}. We use here a definition which subsumes the others.

\tikzset{LPTnode/.style={draw,ellipse,inner sep=2pt,anchor=base,outer sep=0pt}}
\tikzset{LPTtable/.style={draw,rectangle,rounded corners,inner sep=2pt,anchor=base,outer xsep=5pt}}
\begin{figure}
\begin{tikzpicture}[x=2em,y=-2em]
\node[LPTnode] (W) at (0,0) {$W$}; \node[LPTtable,anchor=west] at (W.east) {$\bar w≥w$};
\node[LPTnode] (CP) at (-1,1) {$CP$}; \node[LPTtable,anchor=east] at (CP.west)
  {$\array c c₃p ≥ c₁p ≥ c₃\bar p ≥ c₁\bar p ≥ c₂\bar p\\c₁p ≥ c₂p ≥ c₂\bar p \endarray$};
\node[LPTnode] (P) at (1,1) {$P$}; \node[LPTtable,anchor=west] at (P.east) {$\bar p ≥ p$};
\node[LPTnode] (C) at (1,2) {$C$}; \node[LPTtable,anchor=west] at (C.east) {$c₃ ≥ c₁ ≥ c₂$};
\draw (CP) to node[below,pos=.9] {$\bar w$} (W)to node[below,pos=.15] {$w$}(P)--(C);
\end{tikzpicture}
\caption{An LP-tree equivalent to the set of CP-statements of Example~\ref{exple:CP-thbis}.}
\label{fig:LPtree}
\end{figure}

An LP-tree that is equivalent to the set of CP-statements of Example~\ref{exple:CP-thbis} is depicted on Figure~\ref{fig:LPtree}. More generally, an LP-tree over $ \cal X  $ is a rooted tree with labelled nodes and edges, and a set of preference tables; specifically
\begin{itemize}

\item every node $ N $ is labelled with a set of attributes, denoted $ §Var§(N) $;

\item if $ N $ is not a leaf, it can have one child, or $ | \dom {§Var§(N)} | $ children;

\item in the latter case, the edges that connect $ N $ to its children are labelled with the instantiations in $ \dom {§Var§(N)} $;

\item if $ N $ has one child only, the edge that connects $ N $ to its child is not labelled: all instantiations in $ \dom {§Var§(N)} $ lead to the same subtree;

\item we denote by $ §Anc§(N) $ the set of attributes that appear in the nodes between the root and $ N $ (excluding those at $ N $), and by $ §Inst§(N) $ (resp. $ §NonInst§(N) $) the set of attributes that appear in the nodes above $ N $ that have more than one child (resp. only one child);


\item a conditional preference table $ §CPT§(N) $ is associated with $ N $: it contains local preference rules of the form $ α: ≥ $, where $ ≥ $ is a 
\modifLPTantisym{partial order} over $ \dom {§Var§(N)} $, and $ α $ is a propositional formula over some attributes in $ §NonInst§(N) $.
\end{itemize}

We assume that the rules in $ §CPT§(N) $ define their preorder over $ \dom{§Var§(N)} $ in extension. Additionally, two constraints guarantee that an LP-tree $ φ $ defines a unique preorder over $ \dom{\cal X } $:

\begin{itemize}
\item no attribute can appear at more than one node on any branch of $ φ $; and,


\item at every node $ N $ of $ φ $, for every $ u ∈ \dom{§NonInst§(N)} $,  $ §CPT§(N) $ must contain exactly one rule $ α:≥ $ such that $ u ⊧ α $.

\end{itemize}
Given an LP-tree $ φ $ and an alternative $ o ∈ \dom{\cal X } $, there is a unique way to traverse the tree, starting at the root, and along edges that are either not labelled, or labelled with instantiations that agree with $ o $, until a leaf is reached. Now, given two distinct alternatives $ o, o' $, it is possible to traverse the tree along the same edges as long as $ o $ and $ o' $ agree, until 
%
%
either a leaf node is reached, or
a node $ N $ is reached which is labelled with some $ W $ such that $ o[W] ≠ o'[W] $: in the latter case, we say that $ N $ decides $ \{o,o'\} $.

In order to define $ ≽_φ $ for some LP-tree $ φ $, let $ φ^* $ be the set of all pairs of distinct alternatives $ (o,o') $ such that there is a node $ N $ that decides $ \{o,o'\} $ and the only rule $ α:≥ ∈ §CPT§(N) $ with $ o[§NonInst§(N)] \linebreak[2]= o'[§NonInst§(N)] ⊧ α $ is such that $ o[W] ≥ o'[W] $. Then $ ≽_φ $ is the reflexive closure of $ φ^* $. Note that if there is no node that decides $ \{o,o'\} $, or if the node that decides that pair is labelled with some $ W $ and if the local preference table is such that $ o[W] $ and $ o'[W] $ are incomparable, then $ o ⋈_φ o' $.



\begin{propositionrep} 
\label{prop:LPT=preorder}
Let $ φ $ be an LP-tree over $ \cal X  $, then $ ≽_φ $ as defined above is a 
%
\modifLPTantisym{partial order}.
Furthermore, $ ≽_φ $ is a linear order if and only if 1) every attribute appears on every branch and 2) every preference rule specifies a linear order.
\end{propositionrep}

\begin{proofsketch}
This property is a generalisation of Proposition~1 by \cite{boothetal:rep-irit09}, which was restricted to unary LP-trees with conjunctive conditions and indexed tables. However, it is not difficult to see that their proof still works in this more general settings.
\end{proofsketch}

\begin{appendixproof}
By definition, $ ≽_φ $ is reflexive.
\modifLPTantisym{The fact that it is antisymmetric follows from the antisymmetry of the local preference relations in the conditional preference tables.}
For transitivity, the proof given by \cite{boothetal:rep-irit09} is for a restricted family of LP-trees, so we recast it here for our more general family of LP-trees. Suppose that $ o ≽_φ o'  ≽_φ o'' $ and $ o $, $ o' $, $ o'' $ are distinct. There must be a node $ N $ at which $ \{o,o'\} $ is decided, let $ W $ be the set of attributes that labels $ N $, then $ o[§Anc§(N)] =o'[§Anc§(N)] $, and there is one rule $ α : ≥ $ such that $ o[§NonInst§(N)] = o'[§NonInst§(N)] ⊧ α $ and $ o[W] ≥ o'[W] $. Similarly, let $ N' $ be the node at which $ \{o',o''\} $ is decided, let $ W' $ be the set of attributes that labels $ N' $, then $ o[§Anc§(N')] =o'[§Anc§(N')] $, and there is one rule $ α' : ≥' $ s.t. $ o'[§NonInst§(N')] = o''[§NonInst§(N')] ⊧ α' $ and $ o'[W'] ≥' o''[W'] $.
If $ N = N' $, then $ o[§Anc§(N)] =o'[§Anc§(N)] \linebreak[2]= o''[§Anc§(N)] $, and $ o[W] > o'[W] > o''[W] $ since $ ≥ $ is antisymmetric, thus $ o[W] > o''[W] $ because $ ≥ = ≥' $ is also transitive, hence $ N $ decides $ \{o,o''\} $ and $ o ≽_φ o'' $.
If $ N ≠ N' $, note that both nodes are in the unique branch in $ φ $ that corresponds to $ o' $, so one of $ N $, $ N' $ must be above the other. Suppose that $ N  $ is above $ N' $, then, it must be the case that $ o'[W] = o''[W] $,  and $ o[W] ≠ o'[W] $, thus $ N $ decides $ \{o,o''\} $; moreover, since $ §NonInst§(N) ⊆ §NonInst§(N') $,  $ o[§NonInst§(N)] = o'[§NonInst§(N)] = o''[§NonInst§(N)] ⊧ α $, and $ o[W] ≥ o'[W]  = o''[W]  $; hence $ o ≽_φ o'' $. The case where $ N' $ is above $ N $ is similar.

For the second part of the proposition, suppose first that every attribute appears on every branch and that every preference rule specifies a linear order: we will show that $ ≽_φ $ is antisymmetric and connex. For antisymmetry, consider distinct alternatives $ o, o' ∈ \dom{\cal X} $: because every attribute appears on every branch, there must be a node $ N $, labelled with some $ W ⊆ \cal X $,  that decides $ \{o,o'\} $,  and a unique rule $ α : ≥ $ at $ N $ such that $ o[§NonInst§(N)] = o'[§NonInst§(N)] ⊧ α $; $ ≥ $ must be a linear order over $ \dom W $, so either $ o[W] > o'[W] $ and $ o ≻_φ o' $, or  $ o'[W] > o[W] $ and $ o' ≻_φ o $: $ ≽_φ $ is connex and antisymmetric. For the converse, assuming that either there is some branch where some attribute does not appear, or that there is some rule at some node that does not define a linear order, it is not difficult to define two distinct alternatives that cannot be compared with $ ≽_φ $.
\end{appendixproof}

An LP-tree $ φ $ is said to be \emph{complete} if the two conditions in Proposition~\ref{prop:LPT=preorder} hold, that is, if $ ≽_φ $ is a linear order.

From a semantic point of view, an LP-tree $ φ $ is equivalent to the set that contains, for every node $ N $ of $ φ $ labelled with $ W = §Var§(N) $, and every rule $ α:≥^α_N $ in $ §CPT§(N) $, all CP statements of the form $ α ∧ u ∧ w[W \setminus W^{≠}] | V : w[W^{≠}] ≥ w'[W^{≠}] $, where
\begin{itemize}

\item $ u $ is the combination of values given to the attributes in $ §Inst§(N) $ along the edges between the root and $ N $, and

\item $ w,w' ∈ \dom W $ such that $ w ≥^α_N w' $, and $ W^{≠} $ is the set of attributes on which $ w $ and $ w' $ have distinct values;
and

\item $ V = [\cal X  - (§Anc§(N) ∪ W)] $.

\end{itemize}
This set of statements indicate that alternatives that agree on $ §Anc§(N) $ and satisfy $ u ∧ α $, but have different values for $ §Var§(N) $, should be ordered according to $ ≥^α_N $, whatever their values for attributes in $ V $.


\begin{langlist}

\item[§LPT§] = the language of LP-trees as defined above; we consider that §LPT§\ is a subset of §CP§.\footnote{Strictly speaking, for $ §LPT§ ⊆ §CP§ $ to hold, we can add the possibility to augment every formula in $ §CP§ $ with a tree structure.} 

\end{langlist}

Note that, using the notation defined above, $ §k-LPT§ = §LPT§ ∩ k§-CP§ $ is the restriction of §LPT§\ where every node has at most $ k $ attributes, for every $ k ∈ \N $; in particular, $ §1-LPT§ $ is the language of LP-trees with one attribute at each node; and $ §LPT∧§ = §LPT§ ∩ §CP∧§ $ is the restriction of §LPT§\ where the condition $ α $ in every rule at every node is a conjunction of literals.
Search trees of  \cite{Wilson:ecai04,Wilson:aij11} and LP-trees as defined by \cite{Boothetal:ecai10,LangMenginXia:artint18} are sublanguages of §1-LPT∧§; LP-trees of \cite{FargierGimenezMengin:aaai18} and \cite{BrauningHullermeier:pl12} are sublanguages of §LPT∧§.

\begin{notAAMAS}
More precisely, the table below shows how other definitions of LP-trees are more restrictive than the one given above. Note that all these definitions enforce that preference rules must specify linear orders.
\begin{center}
\renewcommand\vhangle{70}
\begin{tabular}{c|c|c|c|c|}
	              \vh{}                & \vh{loc. partial preorder} & \vh{mult. var.} & \vh{incompl. branches} & \vh{uninst. edge} \\ \hline
	       \cite{Wilson:ecai04}        & ✘                     &        ✘        & ✘                      &         ✘         \\ \hline
	     \cite{Boothetal:ecai10}       & ✘                     &        ✘        & ✘                      &         ✓         \\ \hline
	 \cite{BrauningHullermeier:pl12}   & ✘                     &        ✓        & ✓                      &         ✘         \\ \hline
	  \cite{LangMenginXia:artint18}    & ✘                     &        ✘        & ✘                      &         ✓         \\ \hline
	\cite{FargierGimenezMengin:aaai18} & ✘                     &        ✓        & ✓                      &         ✓         \\ \hline
\end{tabular}
\end{center}
\end{notAAMAS}

We also introduce a very restrictive class of LP-trees, which will turn out to have interesting properties when we look at transformations.

\begin{langlist}
\item[$ §k-LPT§^{§lin§} $ =] the language that contains all \emph{linear} $k$-LP-trees, that is, LP-trees where every node has at most $k$ variables, at most one child, and where all conditional preference rules are \emph{un}conditional.
\end{langlist}

Complete, linear 1-LP trees represent the usual lexicographic orderings.

\ifklptcompat


\ihel{ lexico compatible ne definis pas un langage ... du coup c'est pas une entree tres interessante pour le carte de compilation : comment choisir ce langage si on ne sait pas comment ecrire des formules de ce langage  }

\hel{j ai mis tout ce qui concerne l'algo dans le fichier Algo1, au cas où vous vouliez le reintroduire quelquepart}

Many graphical restrictions that have been proposed in order to enable polytime answers to some queries are in fact particular cases of a more general property which we introduce now. We define a new, parameterized family of languages. Given some language $ \cal L   ⊆ §CP§ $ and $ k ∈ ℕ $, we define:

\ihel{j'ai modifié pour "extends" qui n'etait pas defini (en fait il y a une utre def de extends plus haut, qui n'a rien à voir - j espere ne pas m'etre trompée}
\begin{langlist}
\item[$\cal L  \acycllexk = $] the restriction of $ \cal L  $ to formulas $ φ $ such that there exists some complete LP-tree $ ψ ∈ k§-LPT§$ such that $ o ≽_φ o' $ implies $ o ≽_ψ o' $ ($\psi$ extends $\phi$  to a complete preorder). We say that formulas of $  §CP§\acycllexk $ are \emph{$k$-lexico-compatible}.\footnote{This definition generalizes \textit{conditionally acyclic}  formulas of \cite{Wilson:aij11}, which are the formulas of $ §CP§ \acycllexun $.}

\end{langlist}
\cite{Wilson:aij11} proves that acyclic formulas of §1-CP§\ are 1-lexico-compa\-ti\-ble when they enjoy some local consistency property; it illustrates that $k$-lexico-compatibility is indeed a weak form of acyclicity. We will see that k-lexico-compatibility makes some queries tractable. 


The next result shows that proving that some $ φ ∈ §CP§ $ is $ k $-lexico-compatible, for a fixed $ k $, is not always easy, it generalizes a result by \cite{Wilson:aij11}:

\begin{propositionrep}
For a fixed $ k ∈ ℕ $, checking if a formula $ φ ∈ §CP§ $ is $k$-lexico-compatible is \coNP-complete.
\end{propositionrep}

\begin{proofsketch}
For membership in \coNP: a certificate that some given $ φ $ is not lexico-compatible is a branch of a tree built using the algorithm below where failure occurs; \coNP-completeness can be proved using the same reduction of \pb{3sat} as that used by \cite[Prop. 24] {Wilson:aij11} to prove that checking cuc-acyclicity is \coNP-hard.
\end{proofsketch}

\begin{appendixproof}
For membership in \coNP: a certificate that some given $ φ $ is not lexico-compatible is a branch of a tree built using the algorithm above where failure occurs. \coNP-completeness can be proved using the same reduction of \pb{3sat} used by \cite[Prop. 24] {Wilson:aij11} to prove that checking cuc-acyclicity is \coNP-hard. Consider $ m $ clauses $ C₁, …, Cₘ $ over $ n $ binary attributes $ X₁, …, Xₙ $. Let $ \cal X = \{ X₁, …, Xₙ, Y₀, Y₁, …, Yₘ\} $, where the $ Yᵢ $s are new binary attributes. Define
$$
  φ = \{ l|Yₖ : y_{k-1} ≥ \bar y_{k-1} \mid l ∈ Cₖ, 1 ≥ k ≥ m \} ∪ \{ |Y₀ : yₘ ≥ \bar yₘ \}.
$$
and consider some complete LP tree $ ψ $ over set of attributes $ \cal X $.
Every attribute $ Yₖ $ appears in the free part of at least one rule of $ φ $, thus cannot be at the root of any complete LP tree $ ψ $ that is compatible with $ φ $. On the other hand, any of the $ Xᵢ $'s can be at the root, or at any level, in any branch of $ ψ $. Suppose now that $ C₁ ∧ … ∧ Cₘ $ is satisfiable: let $ u $ be an instantiation of $  X₁, …, Xₙ $ that satisfies this CNF, and consider a branch of $ ψ $ where all the $ Xᵢ $s have the same value as in $ u $: at every node $ N $ in such a branch, for every $ Yₖ $, $ §inst§(N) $  is consistent with the condition of at least one rule which has $ Yₖ $ as free part; therefore, no ordering of the $ Yₖ $s in such a branch can be compatible with condition \ref{cond:LP:comp:Anc} in Proposition \ref{prop:cond:lex-compat}. On the other hand, if  $ C₁ ∧ … ∧ Cₘ $ is unsatisfiable, then it is not difficult to see that it is possible to build $ψ $ in such a way that the conditions of Proposition \ref{prop:cond:lex-compat} are satisfied w.r.t. $ φ $, by taking, for instance, the $ Xᵢ $s for the nodes at the $ n $ first levels of $ ψ $: then, since $ C₁ ∧ … ∧ Cₘ $ is unsatisfiable, in every branch of the tree there must be one clause $ Cₖ $ that is not satisfied by the corresponding instantiation of the $ Xᵢ $s, so none of the conditions of the corresponding rules $ l|Yₖ:  y_{k-1} ≥ \bar y_{k-1} $ is satisfied; then attribute $ Y_{k-1} $ can be chosen for the label of the node at the next level, then $ Y_{k-1} $, and so on…
\end{appendixproof}

\fi

\subsection{GAI decompositions}

We also consider \emph{GAI decompostions} \cite{BacchusG95,GonzalesP04}. 
This framework allows the representation of complete and transitive preference relations by a utility function, additively decomposed as a sum of local utility functions bearing on smaller subsets of attributes. Each local utility function can for instance represent a criterion, the global preference deriving from the additive aggregation of the satisfaction degrees provided by the different criteria.

A GAI decomposition  over a set $ \cal X  $  of finite attributes is defined by a set $\varphi = \{g_{Z_1}, \dots,g_{Z_m}\} $ of  functions bearing on subsets $Z_i$ of $ \cal X  $ and taking their values in $R \cup \{-\infty\}$; for any alternative $ o $, let  $g_φ(o) = \Sigma_{i=1}^m  g_{z_i} (o[Z_i])$. The set~$\varphi$ represents the complete and transitive relation $\mathord{≽_φ}$ in which $o \mathord{≽_φ} o'$ if and only if $g_φ(o) \geq g_φ(o')$. Thus $ ≽_φ $ is a weak order.

The  questions related to the succinctness of  GAI representations depend on the way the local functions are represented  -- and so do all the questions related to the complexity of the operations on such representations. It is generally assumed that each $g_{Z_i}$ is represented by a table that associates to each tuple of the domain of $Z_i$ a  real valued utility and the tuples not present in the table receive the utility $0$.

The most common restriction on the language of GAIs consists in bounding by some integer $ k >0 $ the maximum number of attributes in a same subutility; we denote by $ §GAI§_k $ the corresponding language. In particular, $ §GAI§_1 $ is the language of \emph{Additive Utilities}.

\ifGAIacycl {\color{magenta}
Another restriction that can lead to interesting properties in terms of complexity for §GAI§s is that of acyclicity. A GAI decomposition is in fact an hypergraph over $ \cal X  $; a common definition of acyclicity for hypergraphs is as follows:
\begin{definition}[e.g. \cite{GottlobGrecoScarcello:BordeauxHamadiKohli:book14}]
A jointree for a GAI decomposition $\varphi = \{g_{Z_1}, \dots,g_{Z_m}\} $ is
a tree $ τ $ whose vertices are the $ Z_i$s, and such that, whenever the same attribute $ X ∈ \cal X $ occurs in two subutilities $ g_{Z} $ and $ g_{Z'} $, then $ X $ occurs in every $ Z'' $ on the unique path linking $ Z $ and $ Z' $ in $ τ $
(connectedness condition for X). GAI decomposition $ φ$ is said to be acyclic if it admits a jointree.
\end{definition}
We denote by §GAI\acycl§\ the language of acyclic GAI decompositions.

\paragraph{Expressiveness}
$ §GAI\acycl§ \mexprm §GAI§ $: il suffit de merger des cliques par exples.

\paragraph{Succinctness} À préciser. J'imagine que quand on fusionne des cliques pour passer d'une GAI non acyclique à une GAI acyclique, la taille des tables explose..??... Reste à trouver un exemple de tables pour par exemple $ n $ attributs $ X₁, … X_n $, et $ n $ cliques $ \{X_i,X_{i+1}\} $ pour qu'on ne puisse pas en faire un GAI acyclique sans tout merger...

\begin{proposition}
§GAI§\ is exponentially more succinct than §GAI\acycl§.
\end{proposition}
\begin{proof}
We consider the following §GAI§\ with attributes $\cal X:= \{X_1, \ldots, X_n\}$: the domain $U_i$ of every variable is $\{0,1\}$ and for every $X_i, X_j$ with $i\ne j$ there is a utility function $g_{ij}(X_i, X_j)$ that is $1$ if and only if both attributes $X_i$ and $X_j$ take the value $1$. Call the resulting §GAI§\ $g$ and let $≽$ be its preorder. Clearly, $g$ has size quadratic in $n$. We will show that all acyclic §GAI§s\ representing the same preorder $≽$ as $g$ have exponential size, which will show the proposition.

In a first step, we show that every §GAI§\ $g'$ representing the same preorder $≽$ as $g$ must have for every $i,j\in \{1, \ldots, n\}$ a utility function containing both the attributes $X_i, X_j$. By way of contradition, assume this were not true. Then there are $i,j\in \{1, \ldots, n\} $ such that for $g'= \{g_{Z_1}, \dots,g_{Z_m}\}$ does not have a $Z_\ell$ with $i,j\in Z_\ell$. Then we can partition $\{1, \ldots, m\}$ into sets $I_i, I_j, I_r$ in such a way that all $Z_\ell$ with $\ell \in I_i$ contain $X_i$, all $Z_\ell$ with $\ell \in I_j$ contain $X_j$, and all $Z_\ell$ with $\ell \in I_r$ contain neither $X_i$ nor $X_j$. Define for every alternative $o$ the function $\bar g_i(o)= \sum_{\ell\in I_i} g'_\ell(o[Z_\ell])$, $\bar g_i(o)= \sum_{\ell\in I_i} g'_\ell(o[Z_\ell])$, $\bar g_r(o)= \sum_{\ell\in I_r} g'_\ell(o[Z_\ell])$. Then for all alternatives $o$ we have $g'(o)= \bar g_i(o)+\bar g_j(o) + \bar g_r(o)$. Now consider the alternative $o_{ab}$ that assign the value $a$ to $X_i$, $b$ to $X_j$, and $0$ to all other attributes. Then, $o_{11}≽ o_{00} \sim o_{01}\sim o_{10}$, so up to some scaling, we may assume that $g'(o_{11}) =1 $ and $g'(o_{11}) =g'(o_{01})=g'(o_{10})=0$. We thus get
\begin{align*}
   0=g'(o_{00}) &= \bar g_i(o_{00}) + \bar g_r(o_{00})+\bar g_r(o_{00})\\
   0=g'(o_{01}) &= \bar g_i(o_{01}) + \bar g_r(o_{01})+\bar g_r(o_{01})\\
   0=g'(o_{10}) &= \bar g_i(o_{10}) + \bar g_r(o_{10})+\bar g_r(o_{10})\\
   1=g'(o_{11}) &= \bar g_i(o_{11}) + \bar g_r(o_{11})+\bar g_r(o_{11}).
\end{align*}
Now observe that $\bar g_i$ does not depend on $X_j$, so $\bar g_i(o_{01}) = \bar g_i(o_{00}) $ and $\bar g_i(o_{11}) = \bar g_i(o_{10}) $. Analogously, $\bar g_j(o_{10}) = \bar g_j(o_{00}) $ and $\bar g_j(o_{11}) = \bar g_j(o_{01}) $ and $\bar g_r(o_{01}) = \bar g_r(o_{00}) = \bar g_r(o_{11}) = \bar g_r(o_{10})$. Thus we get 
\begin{align}
   0=g'(o_{00}) &= \bar g_i(o_{00}) + \bar g_r(o_{00})+\bar g_r(o_{00})\label{eq:a}\\
   0=g'(o_{01}) &= \bar g_i(o_{00}) + \bar g_r(o_{01})+\bar g_r(o_{00})\label{eq:b}\\
   0=g'(o_{10}) &= \bar g_i(o_{10}) + \bar g_r(o_{00})+\bar g_r(o_{00})\label{eq:c}\\
   1=g'(o_{11}) &= \bar g_i(o_{10}) + \bar g_r(o_{01})+\bar g_r(o_{00})\label{eq:d}.
\end{align}
Adding (\ref{eq:b}) and (\ref{eq:c}) and subtracting (\ref{eq:d}), we get 
\begin{align*}
   -1=g'(o_{00}) &= \bar g_i(o_{00}) + \bar g_r(o_{00})+\bar g_r(o_{00})
   \end{align*}
   which contradicts (\ref{eq:a}).
   It follows that, as claimed, for every $i,j\in \{1, \ldots, n\}$ the §GAI§\ $g'$ must have a domain $Z_\ell$ containing both $X_i$ and $X_j$.
   
   Now assume that $g'$ is acyclic. The following is well known from graph theory: if in a join tree of $g'$ there is for every pair of attributes $X_i, X_j$ a vertex $Z_\ell$ of the join tree containing both $X_i$ and $X_j$ (see e.g.~\cite[Chapter 12.3]{Diestel}, then the join tree must in fact contain a vertex $Z_\ell$ that contains \emph{all} attributes $Z_\ell$. So in particular, $g'$ must have a utility function $g'_{Z_\ell}$ such that $Z_\ell = \cal X$. But then the table for $g'{Z_\ell}$ must have exponential size which proves the proposition.
\end{proof}

\paragraph{Queries}
\begin{itemize}
\item tractable pour §GAI§\ ⇒ tractable pour §GAI\acycl§ 
\item hard pour $§GAI§₁$\ ⇒ hard pour §GAI\acycl§
\item tractable : undom. extract = algos classique par marginalisation / message passing, ⇒ undom. check et ≻-cut-extract aussi tractable (étant donné $ o $, calculer $ u(o) $, et $ u^\ast / o^\ast $).
\item ?? ≽-cut-extract
\end{itemize}

\paragraph{Transfo}
\begin{itemize}
\item conditioning tractable
\item $\Out$ pour $§GAI§₁$\ ⇒ $\Out$ pour §GAI\acycl§ (conj. / disj.)
\item  lower projection outside language : je pense que l'exple 15 le montre
\item ?? weak / strong opt. project. ⇒  je pense que c'est ok, que les deux projection sont équivalentes pour les GAIs, et qu'on a, si on projette sur $ V $ et si $ U = \cal X - V $:
$$
  v ≽^{↓V}_{§w.opt.§} v' \ ⇔\  v ≽^{↓V}_{§s.opt.§} v'
  \ ⇔\  \max_{u ∈ \dom U} g_φ(uv) ≥ \max_{u ∈ \dom U} g_φ(uv')
$$
Lorsqu'on a un GAI acyclique, il me semble qu'on peut définir en temps polynomial ψ définie sur $ V $ telle que $ g_ψ(v) = \max_{u ∈ \dom U} g_φ(uv) $ - cf par exemple \cite{CooperdeGivrySchiex:stacs20}
\end{itemize}
}\fi 

\section{Expressiveness}
\label{sec:expressiveness} 

This section presents our results on the expressiveness of the various languages introduced above. To this end, let us introduce the way in which we compare different languages.


\tikzset{lang/.style={draw,rectangle,rounded corners,inner sep=3pt,anchor=base,outer sep=0pt}}
\tikzset{langsub/.style={rectangle,rounded corners,inner sep=0pt,anchor=base,outer sep=0pt}}
\tikzset{mexpr/.style={double equal sign distance,->,>={Latex[width=1.5ex,length=1ex]}}}
\newcommand{\drawmsucco}[2]{\drawmsucc#1--#2; }
\def\drawmsucc#1--#2; {
  }
\def\drawmsuccb#1:#2:#3; {\relax
  \draw[dotted,thick,-]  #1 #2node[sloped,allow upside down,pos=0.6,anchor=center] {\raisebox{-4pt}{$\msucc$}} #3;}

\begin{figure}
\begin{center}
\begin{tikzpicture}[x=4em,y=-3.5em]

\node[langsub] (CP) at (-.5,-.8) {§CP§};
\node[langsub] (CPconj) at (-.5,-0.3) {§CP∧§};
\node[langsub] (CPnoimp) at (.5,-.8) {§CP⋫§};
\node[langsub] (CPconjnoimp) at (.5,-0.3) {§CP∧⋫§};
  \drawmsucc(CP)--(CPconj); \drawmsucc(CP)--(CPnoimp);
  \drawmsucc(CPconj)--(CPconjnoimp); \drawmsucc(CPnoimp)--(CPconjnoimp);
\node [lang, fit={(CP) (CPconj) (CPnoimp) (CPconjnoimp)}] (CPclass) {};

\node[lang] (GAI) at (4,0.5) {$ §GAI§ $};
  \draw[mexpr] (CPclass)--(GAI);

\node[langsub,anchor=east] (kCP) at (-1.4,1) {§k-CP§};
\node[langsub,anchor=west] (kCPconj) at (-1.1,1)  {§k-CP∧§};
  \drawmsucc (kCP)--(kCPconj);
\node [lang, fit={(kCP) (kCPconj)}] (kCPclass) {};

\node[langsub,anchor=east] (LPT) at (1,0.5) {$ §LPT§ $};
\node[langsub,anchor=west] (LPTconj) at (1.3,0.5) {$ §LPT∧§ $};
  \drawmsucc (LPT)--(LPTconj);
\node[lang,fit={(LPT) (LPTconj)}] (LPTclass) {};
  \draw[mexpr] (CPclass)--(LPTclass);

\node[langsub,anchor=east] (kLPT) at (-0.2,2) {$ §k-LPT§ $};
\node[langsub,anchor=west] (kLPTconj) at (0.1,2) {$ §k-LPT∧§ $};
  \drawmsucc (kLPT)--(kLPTconj);
\node[lang,fit={(kLPT) (kLPTconj)}] (kLPTclass) {};

  \draw[mexpr] (CPclass) -- (kCPclass); \draw[mexpr] (kCPclass) -- (kLPTclass);
  \draw[mexpr] (LPTclass)--(kLPTclass);

\node[langsub,anchor=east] (kCPnoimp) at (-2.6,2) {§k-CP⋫§};
\node[langsub,anchor=west] (kCPconjnoimp) at (-2.3,2) {§k-CP∧⋫§};
  \drawmsucc(kCPnoimp)--(kCPconjnoimp);
\node[lang, fit={(kCPnoimp)(kCPconjnoimp) }] (kCPnoimpclass) {};

  \draw[mexpr] (kCPclass) -- (kCPnoimpclass);
  \drawmsucc(kCP)--(kCPnoimp); \drawmsucc(kCPconj)--(kCPconjnoimp);

\ifklptcompat
\node[lang] (kLPTcomp) at (1,1) {$ §CP§\acycllexk $ };
	\draw[mexpr] (CPclass) -- (kLPTcomp);
	\draw[mexpr] (kLPTcomp) -- (kLPTclass);
\fi

\node[lang] (kGAI) at (4,1.5) {§k-GAI§};
  \draw[mexpr] (GAI)--(kGAI);

\node[lang] (kLPTlin) at (2.5,2.8) {$ §k-LPT§^{§lin§} $};
  \draw[mexpr] (kLPTclass)--(kLPTlin);
  

\node[langsub,anchor=east] (k-1CP) at (-1.3,3) {$ §(k-\!1\!)-CP§ $};
\node[langsub,anchor=west] (k-1CPconj) at (-1.2,3) {$  §(k-\!1\!)-CP\!∧§$};
  \drawmsucc(k-1CP)--(k-1CPconj);
\node[lang, fit={(k-1CP)(k-1CPconj) }] (k-1CPclass) {};
	\draw[mexpr] (kCPclass) -- (k-1CPclass);

\node[langsub,anchor=east] (k-1CPnoimp) at (-2.5,4) {§(k-\!1\!)-CP⋫§};
\node[langsub,anchor=west] (k-1CPconjnoimp) at (-2.38,4) {§(k-\!1\!)-CP\!∧\!⋫§};
  \drawmsucc(k-1CPnoimp)--(k-1CPconjnoimp);
\node[lang, fit={(k-1CPnoimp)(k-1CPconjnoimp) }] (k-1CPnoimpclass) {};

  \draw[mexpr] (kCPnoimpclass) -- (k-1CPnoimpclass);
  \draw[mexpr] (k-1CPclass) -- (k-1CPnoimpclass);
  \drawmsucc(k-1CP)--(k-1CPnoimp); \drawmsucc(k-1CPconj)--(k-1CPconjnoimp);

\node[langsub,anchor=east] (k-1LPT) at (-0.1,4) {$ §(k-\!1\!)-\!L\!P\!T§ $};
\node[langsub,anchor=west] (k-1LPTconj) at (0.1,4) {$ §(k-\!1\!)-L\!P\!T\!∧§ $};
  \drawmsucc (k-1LPT)--(k-1LPTconj);
\node[lang,fit={(k-1LPT) (k-1LPTconj)}] (k-1LPTclass) {};

	\draw[mexpr] (k-1CPclass) -- (k-1LPTclass); \draw[mexpr] (kLPTclass) -- (k-1LPTclass);
\ifklptcompat
\node[lang] (k-1LPTcomp) at (1,3) {$ §CP§\acycllexkun $ };
	\draw[mexpr] (kLPTcomp) .. controls (1.3,2) .. (k-1LPTcomp); 
    \draw[mexpr] (k-1LPTcomp) -- (k-1LPTclass);
\fi
\node[lang] (k-1GAI) at (4,3.5) {§(k-1)-GAI§};
  \draw[mexpr] (kGAI)--(k-1GAI);
\node[lang] (k-1LPTlin) at (2.5,4.5) {$ §(k-1)-LPT§^{§lin§} $};
  \draw[mexpr] (k-1LPTclass)--(k-1LPTlin);
  \draw[mexpr] (kLPTlin)--(k-1LPTlin);


%

%




\node[lang] (CPnet) at (-1,5) {§CPnet§};
	\draw[mexpr,rounded corners] (k-1CPnoimpclass) |- (CPnet); 
	\drawmsucc(k-1CPconjnoimp)--(CPnet);
\node[lang] (CPnetacycl) at (1,5) {$§CPnet§\acycl$};
	\draw[mexpr] (CPnet) -- (CPnetacycl); 
\ifklptcompat
    \draw[mexpr] (k-1LPTcomp) .. controls (1.3,4) .. (CPnetacycl); 
\fi

\end{tikzpicture}

$ \cal L  $\tikz\draw[mexpr](0,0)--(2em,0);$\cal L '$: $\cal L  $ is strictly more expressive than $ \cal L ' $
\\Boxes contain languages that are equally expressive.
\\For $ k > 2 $.
\caption{Relative expressiveness.}
\label{figure:exp}
\end{center}
\stefan{je pense qu'il serait bien d'expliquer un peu comment il faut lire Fig 2.}\end{figure}


\begin{definition}
Let $ \cal L  $ and $ \cal L ' $ be two languages for representing preorders. We say that 
$ \cal L  $ is \emph{at least as expressive as} $ \cal L ' $, written $ \cal L  \mexpreq \cal L ' $, if every preorder that can be represented with a formula of $ \cal L ' $ can also be represented  with a formula of $ \cal L  $; we write $ \cal L  \mexpr \cal L ' $ if $ \cal L  \mexpreq \cal L ' $ but it is not the case that $ \cal L ' \mexpreq \cal L  $, and say in this case that $ \cal L  $ is strictly more expressive than $ \cal L ' $. We write $ \cal L \mexprm \cal L' $ when the two languages are equally expressive.
\end{definition} 

We reserve the usual ``rounded'' symbols $ ⊂ $ and $ ⊆ $ for (strict) set inclusion, and $ ⊃ $ and $ ⊇ $ for the reverse inclusions. Note that $ \mexpreq $ is a preorder, and obviously $ \cal L  ⊇ \cal L ' $ implies $ \cal L  \mexpreq \cal L ' $.

Figure~\ref{figure:exp} gives a summary of the expressiveness results we show in this section. Note that the fact that acyclicity restricts the expressiveness of CP-nets has been shown in e.g. \cite{Boutilieretal:jair04}.

Let us start exploring the relative expressiveness of different languages.
Clearly, $ §CP⋫§ ⊂ §CP§ $ and $ §CP∧§ ⊂ §CP§ $; however, these three languages have the same expressiveness, because of the following:

\begin{property} \label{prop-mexp-CP}
Given some preorder $ ≽ $, define
$$ φ =
  \{ o[\cal X - Δ(o,o')]:o[Δ(o,o')] ≥ o'[Δ(o,o')] | o ≽ o', o ≠ o'\} 
,$$
where $ Δ(o,o') $ is the set of attributes that have different values in $ o $ and $ o' $, then $ φ ∈ §CP⋫∧§ $, and $ \mathord{≽_φ} = \mathord{≽} $.
\end{property}


A large body of works on CP-statements since the seminal paper by \cite{Boutilieretal:compint04} concentrate on various subsets of §1-CP§. With this strong restriction on the number of swapped attributes, CP-statements have a reduced expressiveness.

\begin{notAAMAS}
\begin{example}
\label{exple-CP-CP1}
Consider two binary attributes $ A $ and $ B $, with respective domains $ \{ a, \bar a\} $ and $ \{b, \bar b\} $. Define preorder $ ≽ $ such that $ ab ≻ \bar a\bar b $, with the two remaining alternatives being incomparable to the former and to each other. This can be represented in §CP§\ with $ φ = \{ ab ≥ \bar a\bar b\} $. But it cannot be represented in §1-CP§, because this would require at least two rules: one to flip the value of $ A $, the other one to flip the value of $ B $; but then there must be one intermediate alternative comparable with $ ab $ and $ \bar a \bar b $.
\end{example}
\end{notAAMAS}

\begin{example}
\label{exple-CP-CP1-tot}
Consider two binary attributes $ A $ and $ B $, with respective domains $ \{ a, \bar a\} $ and $ \{b, \bar b\} $. Define preorder $ ≽ $ such that $ ab ≻ \bar a\bar b ≻ a\bar b ≻ \bar a b $. This can be represented in §CP§\ with $ φ = \{ ab ≥ \bar a\bar b, \allowbreak \bar b:\bar a ≥ a, \allowbreak a\bar b ≥ \bar a b \} $. But it cannot be represented in §1-CP§: $ \{b: a ≥ \bar a, \allowbreak \bar b: \bar a ≥ a, \allowbreak a: b ≥ \bar b, \bar a: \bar b ≥ b\}^* ⊆ φ^* $, but this is not sufficient to compare $ a\bar b $ with $ \bar a b $. The four remaining formulas of §1-CP§\ over these two attributes are $ B:a≥ \bar a $, $ B:\bar a≥a $, $ A:b ≥ \bar b $, $ A:\bar b ≥ b $, adding any of them to $ φ $ yields a preorder which would not be antisymmetric.
\end{example}

Forbidding free parts incurs an additional loss in expressiveness:

\begin{example}
\label{exple-CP1noimp-CP1}
Consider two binary attributes $ A $ and $ B $, with respective domains $ \{ a, \bar a\} $ and $ \{b, \bar b\} $. Define preorder $ ≽ $ such that $ ab ≻ a\bar b ≻ \bar a b ≻ \bar a\bar b $. This can be represented in §1-CP§\ with $ φ = \{ B:a ≥ \bar a, b ≥ \bar b\} $. But the ``tradeoff'' $ a\bar b ≻ \bar a b $ cannot be represented in §1-CP⋫§, any formula of §1-CP⋫§\ that implies it will put some \textit{intermediate} alternative between $ a\bar b $ and $ \bar a b $
\end{example}

However, restricting to conjunctive statements does not incur a loss in expressiveness.

\begin{propositionrep} $ §CP§ = ⋃_{k∈ℕ} §k-CP§ $ and, for every $ k ∈ ℕ $, $ k ≥ 2 $:
\begin{gather*}
  §CP∧§ \:\mexprm\:§CP⋫§ \:\mexprm\:§CP∧⋫§ \: \mexprm\: §CP§ \:\mexpr\: §k-CP§ \:\mexprm\:  §k-CP∧§ \:\mexpr\: §k-CP⋫§ \:\mexprm\: §k-CP∧⋫§
\\
  §k-CP§ \:\mexpr\: §(k-1)-CP§ \ \ \text{ and } \ \ §k-CP∧⋫§ \:\mexpr\: §(k-1)-CP∧⋫§ \:\mexpr\: §CPnet§.
\end{gather*}

\end{propositionrep}

\begin{appendixproof}
That $ §CP§ \mexprm §CP∧§ \mexprm §CP⋫§ \mexprm §CP∧⋫§ $ follows from property~\ref{prop-mexp-CP}.

By definition $ §CP§ ⊃ §1-CP§ ⊃ §1-CP⋫§ ⊃ §1-CP∧⋫§ $ and $ §1-CP§ ⊃ §1-CP∧§ ⊃ §1-CP∧⋫§$, thus $ §CP§ \mexpreq §1-CP§ \mexpreq §1-CP⋫§ \mexpreq §1-CP∧⋫§ $ and $ §1-CP§ \mexpreq §1-CP∧§ \mexpreq §1-CP∧⋫§$. Restricting to conjunction of literals does not induce a loss in expressiveness because, given a statement $ α | V : x ≥ x' $, it is possible to compute a DNF logically equivalent to $ α $, and then consider a set of statements, each statement having one disjunct of the DNF as conditioning part. Example  \ref{exple-CP-CP1-tot} prove that  $ §CP§ ⊐ §1-CP§ $. Example \ref{exple-CP1noimp-CP1} proves that $ §1-CP§ \mexpr §1-CP⋫§ $, it can be generalized to prove that $ §k-CP∧§ \mexpr §k-CP⋫§ $ by considering $ k $ binary attributes $ A₁, …, Aₖ $ instead of $ A $, and the preorder $ a₁…aₖb ≻ a₁…aₖ\bar b ≻ \bar a₁…\bar aₖ b ≻ \bar a₁…\bar aₖ\bar b $, which can be represented in $ §k-CP∧§ $ but not in $ §k-CP⋫§ $.

To prove that $ §k-CP§ \:\mexpr\: §(k-1)-CP§ $, simply consider $ k $ binary attributes $ A₁, …, Aₖ $ and the preorder that contains a single pair: $ a₁…aₖ ≻ \bar a₁…\bar aₖ $, it can be represented in $ §k-CP§ $ with the single statement $ a₁…aₖ ≥ \bar a₁…\bar aₖ $, but not in $ §(k-1)-CP§ $. Note that this statement is in $ §k-CP∧⋫§ $, so it proves that $ §k-CP∧⋫§ \:\mexpr\: §(k-1)-CP∧⋫§ $. The fact that $ §CP-net§ \not\mexpreq §(k-1)-CP§ $ follows from the ``completeness'' condition in the definition of CP-nets: in a CP-net, every attribute must have some local preference rules associated to it, whereas a formula in $§1-CP∧⋫§ $ may consist of one rule only.
\end{appendixproof}

Because an LP-tree can be a single node labelled with $ \cal X  $, and a single preference rule $ ⊤ : ≥ $ where $ ≥ $ can be any 
\modifLPTantisym{partial order}, §LPT§\ can represent any 
\modifLPTantisym{partial order}.
Limiting to conjunctive conditions in the rules is not restrictive.
However, restricting to $ §1-LPT§ $ reduces expressiveness, even if one considers formulas of $ §1-CP§ $ that represent total, linear orders:

\begin{example}
Let
$
  φ = \{ a ≥ \bar a, \bar c|A:\bar b≥b,  \bar ac:\bar b≥b,  ac : b≥\bar b,  a:c≥\bar c, \allowbreak \bar a|B:\bar c≥c \}.
$
This yields the following linear order:
$
  abc ≽_φ a\bar bc ≽_φ a\bar b\bar c ≽_φ \bar a\bar b\bar c ≽_φ ab\bar c ≽_φ \bar ab\bar c ≽_φ \bar a\bar bc ≽_φ \bar abc. 
$
No $ ψ ∈ §1-LPT§  $ can represent it: $ A $ could not be at the root of such a tree because for instance $  a\bar b\bar c ≽_φ \bar a\bar b\bar c  $ and $ \bar a\bar b\bar c ≽_φ ab\bar c $; neither could $ C $, since $ a\bar bc ≽_φ a\bar b\bar c $ and $ \bar ab\bar c ≽_φ \bar a\bar bc $; and finally $ B $ could not be at the root either, because $ abc ≽_φ a\bar bc $ and $ \bar a\bar b\bar c ≽_φ ab\bar c $.
\end{example}

\begin{propositionrep} $ §LPT§ = ⋃_{k∈ℕ} §k-LPT§ $ and, for every $ k ∈ ℕ $:
\begin{gather*}
  §CP§ \;\modifLPTantisym{\mexpr}\; §LPT§ \;\mexprm\; §LPT∧§ \;\mexpr\; §k-LPT§  \;\mexprm\; §k-LPT∧§  \;\mexpr\;  §(k-1)-LPT§\\
  §k-CP§ \;\mexpr\; §k-LPT§  \;\mexpr\; §k-LPT§^{§lin§}  \;\mexpr\;  §(k-1)-LPT§^{§lin§}.
\end{gather*}
\end{propositionrep}

\begin{appendixproof}
\modifLPTantisym{LP trees can only represent antisymmetric preorders, so $ §LPT§ $ is strictly less expressive than $ §CP§ $.}
That $ §k-LPT∧§ \mexpreq §k-LPT§$ follows from the fact that the condition of every CP-statement in the set of CP-statements that correspond to some $ §k-LPT§ $ can be represented with a DNF. To see that $ §(k-1)-LPT§ \not\mexpreq §k-LPT§$ for every $ k ≥ 1 $, consider some §k-LP§-tree $ φ $ with $ k $ attributes $ X₁, …, Xₖ $ at the root, and a linear order with $ x₁…xₖ $ as top element, and $ \bar x₁…\bar xₖ $ as second best element: then $ o ≽_φ o' $ for every pair of alternatives $ (o,o') $ such that $ o[X₁…Xₖ] = x₁…xₖ $ and $ o'[X₁…Xₖ] = \bar x₁…\bar xₖ $; no LP-tree in $ §(k-1)-LPT§ $ can represent that. Note that we can choose φ to be linear, so that proves that $ §(k-1)-LPT§^{§lin§} \not\mexpreq §k-LPT§^{§lin§}$ too.

To show that $ §k-LPT§^{§lin§} \not\mexpreq §k-LPT§ $, consider an LP-tree with attributes $ X₁ … Xₖ $ at the root, with $ \card{\dom{\{X₁…Xₖ\}}} $ children, where every child is labelled with binary attribute $ Y $, and at least two children order $ y $ and $ \bar y $ differently: no linear $k$-LP tree can represent the same order.

That $ §k-CP§ \;\mexpreq\; §k-LPT§ $ follows from the remark below Proposition~\ref{prop:LPT=preorder} that describes a set of CP-statements equivalent to a given LP-tree $ φ$: it is not difficult to check that if every node in $ φ $ has at most $ k $ attributes, then the corresponding CP-statements are all in $ §k-CP§ $. To prove that $ §k-LPT§ \not\mexpreq §k-CP§ $, consider a CP-net φ over $ k+1 $ binary attributes $ X₁, …, X_{k+1}$, with $ x_i ≥ \bar x_i $ for every $ 1 ≤ i ≤ k+1 $ (thus the CP-net has no edge): clearly $ φ ∈ §1-CP§ ⊆ §k-CP§ $. Consider now some LP-tree ψ with $ j ≤ k $ attributes at the root; w.l.o.g. we can assume that these attributes are $ X₁, …, X_j $; then the CPT at the root of ψ must contain the preorder over $ X₁, …, X_j $ defined by the set of CP-statements $ \{ x_i ≥ \bar x_i | 1 ≤ i ≤ j \} $. But then $ x₁…x_j\bar x_{j+1}…\bar x_{k+1} ≻_ψ \bar x₁…\bar x_j x_{j+1}… x_{k+1} $, whereas $ x₁…x_j\bar x_{j+1}…\bar x_{k+1} ⊁_φ \bar x₁…\bar x_j x_{j+1}… x_{k+1} $. Since $§LPT§ \mexpreq §CP§ \mexpr §k-CP§ $, we can conclude too that $§LPT§ \mexpr §k-LPT§ $.

\end{appendixproof}

\ifklptcompat
Finally, note that k-lexico-compatibility is a weaker restriction than being a $ k $-LP-tree. 

\begin{proposition}
For every $ k ∈ ℕ $: $ §CP§\mathord{\acycllexk} \mexpr §CP§\mathord{\acycllexkun} $, and $ §CP§\mathord{\acycllexk} \mexpr §k-LPT§ $.
\end{proposition}

\cite{Wilson:aij11} proves that $ §1-CP§\mathord{\acycl} ⊆ §CP§\mathord{\acycllexun} $. Whether this property can be generalized, with an appropriate definition of \textit{$k$-acyclicity}, is left for future work.
\fi


Finally, because GAI decompositions are restricted to the representation of complete preference relations, their expressiveness is lower than the one of the general CP language; the latter can represent any transitive relation, so  CP is strictly more expressive than GAI. Subclasses of the CP language may be incomparable with GAI. The same line of reasoning applies when comparing GAI and complete lexicographic trees: both target the representation of complete orders, but the former language allows the representation of any complete preorder, while the latter can represent linear orders only (antisymmetry is required). It follows that GAI are strictly more expressive than complete LP trees.
\stefan{how is this sentence supposed to end?}
\jerome{J'ai supprimé la phrase.}
We summarize these observations below.

\begin{proposition}
 $ §CP§ \mexpr §GAI§  \mexpr §complete-LPT§ $.
\end{proposition}

The second source of limitations on expressiveness comes from the bounding of the number of attributes present in the expression of local preferences. 
Using the same counter example as those used for showing that CP is strictly more expressive than $k$-CP, one can show that  GAI and $k$-CP restrictions are incomparable in terms of expressiveness.

\begin{propositionrep}
	For every $ k ∈ ℕ $: $ §GAI§_{k+1} \mexpr §GAI§_k$, and $§(k-1)-CP§  \not\mexpr §GAI§_{k} $.
\end{propositionrep}

\begin{appendixproof}
	We first prove the first statement.
	$§GAI§_{k+1} \mexpreq §GAI§_k$ is clear, since $§GAI§_{k+1} \supseteq §GAI§_k$, so it suffices to show that the increase in expressiveness is strict.
	
	Fix the set of attributes $\cal X = \{X_1, \ldots, X_{k+1}\}$ and set as the domain of each attribute $\dom{X_i}=\{0,1\}$. For every set $A\subseteq \cal X$, define the indicator function $I_A(X_1, \ldots , X_{k+1})$ as the function that, given as input an alternative $o\in \dom{\cal X}$, returns $1$ if for all $X\in A$ we have $o(X)=1$ and $0$ otherwise. Set $\varphi = \{I_{\cal X}\}$, then we have that $g_\varphi(o) = I_{\cal X}(o)$. For every $A$ define $o_A$ to be the alternative that is $1$ on exactly the attributes in $A$. Then $g_\varphi$ induces the total preorder $\succeq$ in which $o_{\cal X}$ strictly dominates all other alternatives, whereas for all pairs $o, o'$ both different from $o_{\cal X}$ we have $o \sim o'$.
	
	Clearly, $\succeq$ is expressed in $§GAI§_{k+1}$. We claim that it cannot be expressed in $§GAI§_{k}$. To this end, assume that this were wrong, then there is a set $\varphi = \{g_{Z_1}, \dots,g_{Z_m}\} $ of real valued functions bearing on strict subsets $Z_i$ of $ \cal X  $ such that $g_\varphi$ induces the order $\succeq$ on $\dom{\cal X}$. Without loss of generality, assume that for every $A\subset \cal X$ the set $\cal X$ contains exactly one function $g_A$. It will be convenient to represent $g_{A}$ as a weighted sum of indicator functions.
	
	We use the following representation result for functions from $\{0,1\}^{\ell}\rightarrow \mathbb{R}$ whose proof can e.g.~be found in~\cite[Section 13.2]{CramaH11}.
	
	\begin{lemma}
		For every function $f:\{0,1\}^{\ell} \rightarrow \mathbb{R}$ with $\ell\in \mathbb{N}$ and in variables $x_1', \ldots , x_\ell'$, there are coefficients $c_A\in \mathbb{R}$ for $A\subseteq \{x_1', \ldots, x_\ell'\}$ such that
		\begin{align*}
			f(x_1', \ldots, x_\ell') = \sum_{A\subseteq \{x_1', \ldots, x_\ell'\}} c_A I_A(x_1', \ldots, x_\ell').
		\end{align*}
	\end{lemma}
	
	Applying this to the utility functions, it follows directly that,
	for every $A\subset \cal X$, there are coefficients $\lambda_{A,B}$ for $B\subseteq A$ such that for all alternatives $o\in \dom{\cal X}$ we have 
	\[g_A(o) = \sum_{B\subseteq A} \lambda_{A,B} I_{B}(o).\]
	We get by summing the $g_A$ that there are coefficients $\lambda_B$ such that for all $o\in \dom{\cal X}$ 
	\begin{equation}\label{eqn:GAI}g_\varphi(o) = \sum_{B\subset \cal X} \lambda_B I_{B}(o).\end{equation}
	By subtracting values in some of the $g_A$, we may assume w.l.o.g.~that $g_\varphi(o)= 0$ for all $o\ne o_{\cal X}$.
	
	We claim that, for all $B\subset \cal X$, we have $\lambda_B = 0$.
	We show this by induction on the size of $B$. For $B=\emptyset$, we have with (\ref{eqn:GAI}) that $0 = g(o_\emptyset) = \lambda_{\emptyset} I_\emptyset(o_\emptyset) = \lambda_\emptyset$.
	For non-empty $B\subset \cal X$, we have $g_\varphi(o_B) = \sum_{C\subset \cal X} \lambda_C I_{C}(o_B) = \sum_{C\subseteq B} \lambda_C I_{C}(o_B)$. However, by the induction hypothesis, we know that for $C\subset B$ we have $\lambda_C = 0$, so $0= \lambda_B I_{B}(o_B) = \lambda_B$.
	
	Plugging the $\lambda_B$ into (\ref{eqn:GAI}), we get that $g_\varphi(o_{\cal X}) = 0$ which contradicts the assumption that in $\succeq$ the alternative $o_{\cal X}$ strictly dominates all others.
	
	For the second statement, consider $k$ binary attributes $A_1, \ldots, A_k$ such that $a_1\ldots a_k \succ \bar a_1\ldots \bar a_k$. Extend this to an arbitrary complete preference relation such that for all other alternatives $o$ we have $o \succ a_1\ldots a_k$. Clearly, any such order can be expressed as a $§GAI§_{k}$ by simply giving all alternatives $o$ a utility that yields this order in a single $k$-ary function $g_{\{A_1, \ldots, A_k\}}$. We claim that this order cannot be expressed by a $§(k-1)-CP§$. Assume this were false, so there is a set of preference statements defining the order and in which the set of swapped attributes never contains more than $k-1$ attributes. In particular, there is such a statement $\alpha \mid V: w \ge w'$ that sanctions $a_1\ldots a_k \succ \bar a_1\ldots \bar a_k$ (this comparison cannot be obtained by transitivity, since all other alternatives have a utility that is strictly greater than that of $ a_1\ldots a_k $). By assumption $w$ cannot contain all attributes, so there is one attribute, say w.l.o.g.~$A_1$ that does not appear in $w$. If $A_1$ is not 
	in $V$, then, by definition, applying the statement cannot swap the value of $A_1$, so it cannot justify $a_1\ldots a_k \succ \bar a_1\ldots \bar a_k$. So $A_1$ must appear 
	in $V$. 
	Then $ A_1 \notin §Var§(\alpha) $, thus the statement also sanctions $a_1a_2\ldots a_k \succ a_1\bar a_2\ldots \bar a_k$ which contradicts the order we want to define.
	So as we claimed, $\succ$ is not defined by any $§(k-1)-CP§$.
\end{appendixproof}

\section{Succinctness}
\label{sec:succinctness} 

Another criterion is the relative sizes of formulas that can represent the same preorder in different languages. This section details our results about the succinctness of the various languages introduced above.

Cadoli et al.~\citep{CadoliDoniniLiberatoreSchaerf:jair00} study the space efficiency of various propositional knowledge representation formalisms. An often used definition of succinctness \cite{GogicKautzPapadimitriouSelman:ijcai95,DarwicheMarquis:jair02} makes it a particular case of expressiveness, which is not a problem when comparing languages of same expressiveness. However, we study here languages with very different expressiveness, so we need a more fine grained definition:

%
%

\begin{definition}
Let $ \cal L  $ and $ \cal L ' $ be two languages for representing preorders. We say that $ \cal L  $ is \emph{at least as succinct as} $ \cal L ' $, written $ \cal L  \msucceq \cal L ' $, if there exists a polynomial $ p $ such that for every $ φ' ∈ \cal L ' $, there exists  $ φ ∈ \cal L  $ that represents the same preorder as $ φ' $ and such that $ |φ| < p(|φ'|) $.%
\footnote{Where $ |φ|= ∑_{α | V : w ≥ w'  ∈ φ} (|α| + |V| +2|§Var§(w)|) $, with $ |α| = $ the number of connectives plus the number of atoms of $ α $.}
Moreover, we say that~$ \cal L  $ is \emph{strictly more succinct than} $ \cal L ' $, written $ \cal L  \msucc \cal L ' $, if
$ \cal L  \msucceq \cal L ' $
and
for every polynomial $ p $, there exists $ φ ∈ \cal L  $ such that:
\begin{itemize}
\item there exists $ φ' ∈ \cal L ' $ such that $ ≽_φ = ≽_{φ'} $, but
\item for every $ φ' ∈ \cal L ' $ such that $ ≽_φ = ≽_{φ'} $,  $ |φ'| > p(|φ|) $.    
\end{itemize}
\end{definition}

With this definition, $ \cal L  \!\!\msucc\!\! \cal L ' $ if every formula of $ \cal L ' $ has an equivalent formula in $ \cal L  $ which is ``no bigger'' (up to some polynomial transformation of the size of $ φ $), and there is at least one sequence of formulas (one formula for every polynomial $ p $) in $ \cal L  $ that have equivalent formulas in $ \cal L ' $ but necessarily ``much bigger''.\footnote{When $ \msucc $ is defined as the strict counterpart of $ \msucceq $, it can happen that $ \cal L \!\!\msucc\!\! \cal L' $ even if there is no real difference in representation size in the two languages, but $ \cal L \!\!\mexpr\!\! \cal L' $.}

\begin{proposition}The following hold, for languages $ \cal L $, $ \cal L' $, $ \cal L'' $:
\begin{itemize}
\item if $ \cal L ⊇ \cal L' $ then $ \cal L \msucceq \cal L' $;
  and if $ \cal L \msucceq \cal L' $, then $ \cal L \mexpreq \cal L' $;
\item if $ \cal L  \msucc \cal L ' $ then $ \cal L  \msucceq \cal L ' $ and $ \cal L '  \not\msucceq \cal L  $;
\item if $ \cal L \mexprm \cal L' $, the reverse implication holds:
  \begin{itemize}
  \item[] if $ \cal L  \msucceq \cal L ' $ and $ \cal L '  \not\msucceq \cal L  $ then $ \cal L  \msucc \cal L ' $
  \end{itemize}
  (otherwise, it might be that $ \cal L '  \not\msucceq \cal L  $ because $ \cal L '  \not\mexpreq \cal L  $);
\item if $ \cal L ⊇ \cal L' $ and $ \cal L' \msucc \cal L'' $, then $ \cal L \msucc \cal L'' $.
\end{itemize}
\end{proposition}

Restricting the conditioning part of CP statements to be conjunctions of literals leads to a loss in succinctness.

\begin{example}
Consider $ 2n+1 $ binary attributes $ X₁, X₂, …, Xₙ, Y₁, Y₂, \allowbreak …, \allowbreak Yₙ, Z $, and let $ φ $ contain $ 2n+2 $ unary CP-statements with no free attribute: $ (x₁∨y₁) ∧ (x₂∨y₂) ∧ … ∧ (xₙ∨yₙ) : z ≥ \bar z $, $ ¬[ (x₁∨y₁) ∧ (x₂∨y₂) ∧ … ∧ (xₙ∨yₙ) ] : \bar z ≥ z $ and $ \bar xᵢ ≥ xᵢ $ and $ \bar yᵢ ≥ yᵢ $ for every $ i ∈ \{1, …, n\} $. Then $ φ ∈ §1-CP⋫§ $, but $ φ $ is not in conjunctive form. A set of conjunctive CP-statements equivalent to $ φ $ has to contain all $ 2ⁿ $ statements of the form $ μ₁μ₂…μₙ : z ≥ \bar z $ with $ μᵢ = xᵢ $ or $ μᵢ = yᵢ $ for every $ i $.
\end{example}

Also, free attributes enable the succinct representation of the relative importance of some attributes over others; disabling free attributes thus incurs a loss in succinctness:

\begin{example}
Consider $ n+1 $ binary attributes  $ X₁, X₂, …, Xₙ, Y $, let $ U = \{X₁, X₂, …, Xₙ\} $, and let $ φ = \{ U | y ≥ \bar y \} $. Then $ φ^* = \{ (uy, u'\bar y) | u, u' \allowbreak ∈ \dom U\}  $, and $ φ^* $ is equal to its transitive closure, so, if $ o  ≠ o' $, then $ o ≽_φ o' $ if and only if $ o[Y] = y $ and $ o'[Y] = \bar y $. This can be represented, without free attribute, only with formula $ ψ $ that contains, for every $ V ⊆ U $ and every $ v ∈ \dom V $, the statement $ vy ≥ \bar v\bar y $, where $ \bar v $ denotes the tuple obtained by inverting all values of $ v $. For every $ 0 ≤ i ≤ n $ there are $  n \choose i $ subsets of $ V $ of size $ i $, with $ 2^i $ ways to choose $ v ∈ \dom V $, thus $ ψ $ contains $ ∑₀ⁿ {n \choose i} 2^i = 3ⁿ $ statements.
\end{example}

Restricting to CP-nets yields a further loss in succinctness, as the next example shows:

\begin{example}
Consider $ n+1 $ binary attributes $ X₁, X₂, …, Xₙ, Y $, and let  $ φ $ be the $ §1-CP⋫∧§ $ formula that contains the following statements: $ xᵢ ≥ \bar xᵢ $ for $ i = 1, …, n $; $ x₁x₂…xₙ:y≥\bar y $; $ \bar xᵢ : \bar y≥y $ for $ i = 1, …, n $. The size of φ is linear in $ n $. Because preferences for $ Y $ depend on all $ Xᵢ $'s, a CP-net equivalent to $ φ $ will contain, in the table for $ Y $, $ 2ⁿ $ CP statements.
\end{example}

\begin{proposition}
The following hold:
\begin{itemize}
\item $\cal L  \msucc \cal L∧  $ for every $ \cal L $ such that $ §1-CP⋫§ ⊆ \cal L  ⊆ §CP§ $;
\item $\cal L  \msucc \cal L\mathord{⋫}  $ for every $ \cal L $ such that $ §1-CP∧§ ⊆ \cal L  ⊆ §CP§ $;
\item $ §1-CP⋫∧§ \msucc §CPnet§ $.
\end{itemize}
\end{proposition}

\ihel{est ce que vous pouvez revoir ce qui suit et la cette demo ? Bon, aprs, je sais pas si le resultat est tres interessant, vu que les langages sont d expressivité differente, on peut enlever si vous voulez}

We have seen that any complete preorder, and in particular the preference captured by any complete LP-tree  can be represented by a GAI.  This representation comes with no increase in size
.

\begin{propositionrep}
Any 
complete LPT can be transformed in polytime and space into an equivalent GAI.
\end{propositionrep}

\begin{appendixproof}
A complete LP tree $ φ $ induces a linear order over $ \domX $, thus we can define the \emph{rank} of alternative $ o $ w.r.t. $ ≽_φ $: $ §rank§(φ,o) = 1 + $ the number of alternatives strictly preferred to $ o $, so that the most preferred alternative has rank 1, the least preferred has rank $ \card{\domX} $:
$$
  §rank§(φ,o) = 1+\card{\{o' \in \domX \mid o' ≻_φ o\}}.
$$
\cite{FargierGimenezMenginNguyen:mpref22} 
explain how $ §rank§(φ,o) $ can be decomposed as a weighted sum of ``local'' ranks associated to the nodes of $ φ $:
\begin{multline}
   §rank§(φ,o) = 1 + \label{eq:rank-decomp}
   \sum_{N ∈ §nodes§(φ)}^{α:≥ ∈ §CPT§(N)}
   ⟦o[§Inst§(N)] \!=\! §inst§(N) ∧ o ⊨ α ⟧ \\\nonumber
   × \big( r(≥,o[§Var§(N)]) - 1\big) 
   × \card{\dom{§Desc§(N)}} 
\end{multline}
where :
\begin{itemize}

\item  $ §nodes§(φ) $ denotes the set of nodes of $ φ $;

\item $ ⟦o[§Inst§(N)] = §inst§(N) ∧ o ⊨ α ⟧ $ is an indicator function, that equals 1 when the condition $ o[§Inst§(N)] = §inst§(N) ∧ o ⊨ α $ is true; that is, when $ N $ is on the branch of $ φ $ that corresponds to $ o $, and $ α:≥ $ is the rule that orders at $ N $ alternatives that have same values as $ o $ for the attributes in the ancestor nodes of $ N $; and equals 0 otherwise;


\item $ r(≥,o[§Var§(N)]) $ denotes the rank in $ \dom{§Var§(N)} $ with respect to $ ≥ $ of the instantiation given by $ o  $ to $ §Var§(N) $; so that $ r(≥,o[§Var§(N)]) - 1 $ is the number of subtrees rooted at children of $ N $ that are less preferred than $ o $ at $ N $;

\item $ §Desc§(N) = \cal X  - (§Anc§(N)∪§Var§(N)) $ is the set of attributes that appear below $ N $ in that branch, so that $ \card{\dom{§Desc§(N)}} $ is the number of instantiations that are ``contained" in every subtree of $ φ $ rooted at any one child of $ N $.

\end{itemize}
Thus we can define, for every node $ N $ of $ φ $, and every rule $ α:≥ ∈ §CPT§(N) $, a sub-utility $ u_{N,α} $ as follows:
$$ \arraycolsep1pt
  u_{N,α}(o) = \left\{\begin{array}{l}
      \big( r(≥ ,o[§Var§(N)]) \!-\! 1\big) \!×\! \card{\dom{§Desc§(N)}}
        \text{ if } o[§Inst§(N)] \!=\! §inst§(N) \text{ \& } o ⊨ α \\
      0 \text{ otherwise } 
  \end{array}\right.
$$
and define a utility $ u_φ $ that orders the alternatives as $ φ $ as follows:
$$
  u_φ = - \sum_{N ∈ §nodes§(φ)}^{α:≥ ∈ §CPT§(N)} u_{N,α}
$$
The number of non-null entries in the table of every $ u_{N,α} $ is equal to $ \dom{§Var§(N)} - 1 $, which also corresponds to the space needed to represent the linear order $ ≥ $ of the rule $ α:≥ $. Assuming that $ §Var§(N) $ (resp. $ §Desc§(N) $) contains $ p $ (resp. $ q $) attributes, the largest entry cannot be larger than $ d^{p+q} ≤ d^n $, where $ d $ is the size of the largest attribute domain, so the number of digits needed for representing the non-null values is polynomial in $ n $ and $ d $. Thus the size of the representation of $ u_φ $ is polynomial in the size of $ φ $.
\end{appendixproof}

\section{Queries}               
\label{sec:queries}

Table~\ref{table:queries} gives an overview of the tractability of the queries that we study in this section. We begin this section with the two queries that have generated most interest in the literature on CP statements.

\begin{table}
\newcommand{\entetecol}[1]{\multicolumn1{c}{\begin{rotate}{70} \parbox[t]{5.2em}{#1}\end{rotate}}}
\newcommand{\entetecolh}{\multicolumn{1}{c}{\vrule height 5.4em width0pt}}
\tabcolsep2pt
\newbox\mybox\setbox\mybox=\hbox{\NP}
\newcommand{\mycwd}{\hspace*{\wd\mybox}}
\def\P{✔}\def\NPh{✘\!∘}\def\NPc{✘}\def\PSPc{\ensuremath{✘\!✘}}\def\PSPh{\ensuremath{✘\!✘\!∘}}\def\sPh{\#✘\!∘}\newcommand\sPc{\#✘}

\begin{center}
\begin{tabular}{|l|c|c|c||c|c|c|c|c|c||c|c|}
\entetecolh  & \entetecol{§GAI§}& \entetecol{§GAI§$_k$}& \entetecol{§GAI§$_1$}&   \entetecol{§CP§} &  \entetecol{§1-CP⋫§}  &  \entetecol{§1-CP⋫∧§}   & \entetecol{§CP§net} & \entetecol{§CP§net\acycl} & \entetecol{§CP§net\acyclpoly}  & \entetecol{§LPT§}  \\ \hline
& & & \mycwd & \mycwd & \mycwd & \mycwd & \mycwd & \mycwd & \mycwd \\[-1em] 
\pb{consistency}                     & & & & \PSPc & \PSPc & \PSPc &      &  ⊤   & ⊤  & \modifLPTantisym{⊤} \\ \hline 
\pb{R-comparison}, $R∈\{≽,≻,⋈\}$         & \nouv {\P}& \nouv {\P}& \nouv {\P}& \PSPc & \PSPc & \PSPc & \NPh & \NPc & \P & \P \\ \hline 
\pb{$∼$-comparison}                      & \nouv {\P}& \nouv {\P}& \nouv {\P} & \PSPc & \PSPc & \PSPc &      &  ⊥   & ⊥  & \P \\ \hline 
\pb{equivalence}                         & & & & \PSPc & \NPh  & \NPh  &    &  \P  &  \P & \NPh \\ \hline 
\pb{top-$p$}                            & \nouv {\P}& \nouv \P& \nouv \P&       &       &       &      &  \P  & \P & \P \\ \hline 
\pb{undominated check}      & \nouv{\NPc} & \nouv{\NPc} & \nouv{\P} & \PSPc & \PSPc & \PSPc &      &  \P  & \P & \P \\ \hline 
\nouv{\pb{undominated extract}}                     & \nouv{\NPh}& \nouv{\NPh} &\nouv{\P} 
  & & & &      &   \nouv{\P}  & \nouv{\P} & \nouv{\P} \\ \hline 
\pb{$≽$-cut extraction} & \nouv{\NPc}& \nouv{\NPc}& \nouv{\P} &  \P   &  \P   &  \P   &  \P  &  \P  & \P & \P \\ \hline 
\pb{$≻$-cut extraction}                  & \nouv{\NPc} & \nouv{\NPc} & \nouv{\P} & \PSPc & \PSPc & \PSPc &      &  \P  & \P & \P \\ \hline 
\pb{$≻$-cut counting}                    & \nouv{\sPc} & \nouv{\sPc} & \nouv{\sPc} & \PSPh & \PSPh & \PSPh & \nouv{\sPh} &   \nouv{\sPh}   &    & \P \\ \hline 
\end{tabular}
\end{center}

\medskip\raggedright
Each column corresponds to one sublanguage of §CP§.
They are sorted in order of decreasing expressiveness from left to right, except when columns are separated by double lines.
For each query and sublanguage:
⊤ = always true for the language; ⊥ = always false for the language;
$ \P = $ polytime answer; $\NPc$ = \NP/\coNP-complete query; $\NPh = $ \NP/\coNP-hard query;
\sPc = \#P-complete query; \sPh = \#P-hard query;
$ \PSPc = $ \PSPACE-complete query;
$ \PSPh = $ \PSPACE-hard query.


\caption{\label{table:queries} Complexity of queries.}
\end{table}

\subsection{Consistency}

Knowing that a given $ φ ∈ §CP§ $ is consistent (that is, that $ ≽_φ $ is antisymmetric) is valuable, as it makes several other queries easier. It also gives some interesting insights into the semantics of $ φ $. The following query has been addressed in many works on CP statements:

\defquery{Consistency} Given $ φ $, is $ φ $ consistent?!

\cite{Boutilieretal:jair04} prove that when its dependency graph $ D_φ $ is acyclic, then a CP-net $ φ $ is consistent. This result has been extended by \cite{DomshlakBrafman:kr02,Brafmanetal:jair06,Wilson:aij11}, who give weaker, sufficient syntactical conditions that guarantee that a locally consistent set of unary, conjunctive CP statements is consistent%
\ifklptcompat; more generally, by definition of $k$-lexico-compatibility, every formula of $ §CP§\acycllexk $ is consistent (since it is compatible with a complete LP-tree). \else
.
\fi
\cite[Theorem 3 and 4]{Goldsmithetal:jair08} prove that \pb{consistency} is \PSPACE-complete for §1-CP⋫∧§.
\modifLPTantisym{We have already seen that the preorder defined by any LP tree is antisymmetric.}



\subsection{Comparing alternatives}
A basic question, given a formula $ φ $ and two alternatives $ o , o' $ is: how do $ o $ and $ o' $ compare, according to $ φ $? Is it the case that $ o ≻_φ o' $, or $ o' ≻_φ o $, or $ o ⋈_φ o' $, or $ o ∼_φ o' $? We define the following query, for any relation $ R ∈ \{≻,≽,∼,⋈\} $:

\defquery{$R$-comparison} Given formula $ φ $, alternatives $ o ≠ o' $, is it the case that $ o R_φ o' $?!

\begin{notAAMAS}
Since the four relations of interest are defined from $ ≽_φ $, in order to decide which is the case, it is sufficient to be able to check if $ o ≽_φ o' $ and to check if $ o' ≽_φ o $. Conversely, if we can answer queries about $ ≻_φ $, and about one of $ ⋈_φ $ or $ ∼_φ $, then we can decide when $ o ≽_φ o' $ holds. For languages where $ ≽_φ $ is always antisymmetric, then \pb{-comparison} is trivial (it is true when $ o = o' $) and \pb{$≽$-comparison} and  \pb{$≻$-comparison} have the same complexity. Furthermore, for languages where $ ≻_φ $ is a linear order, then $ ≽_φ $ is antisymmetric and, moreover, \pb{$⋈$-comparison} is trivially false.
\end{notAAMAS}

For LP-trees, in order to compare alternatives $ o $ and $ o' $, one only has to traverse the tree from the root downwards until a node that decides the pair is reached, or down to a leaf if no such node is encountered: in this case $ o $ and $ o' $ are incomparable. Note that checking if a node decides the pair, and checking if a rule at that node applies to order them, can both be done in polynomial time. \nouv{For generalized additive utilites, two alternatives can be compared by computing their utilities, which is tractable.}

\begin{proposition}
\pb{$R$-comparison} is in \P\ for §LPT§\ and for §GAI§\ for all relations $ R \!∈ \!\{≻, \! ≽,\! ∼,\! ⋈\} $.
\end{proposition}

\begin{commentaire}
The complexity of comparisons has been studied by \cite{Boutilieretal:jair04} for CP nets, by \cite{Goldsmithetal:jair08} for §1-CP⋫§\ and by \cite{Wilson:aij11} for §1-CP∧§. \cite{Goldsmithetal:jair08} propose a simple non-deterministic algorithm to prove membership in \PSPACE\ of \pb{$≽$-comparison}; we rewrite the algorithm here for our more general preference statements:

\begin{algo}{$≽$-comparison. Input: $ o, o', φ $}
\item Repeat:
	\begin{enumerate}
	\item guess $ o'' $, $ α|V:w≥w' ∈ φ $;  $ Y ← \cal X  -(U∪V∪W) $;
	\item if $ α|V:w≥w' ∈ φ $  sanctions $ (o,o'') $: $ o ← o'' $;
	\end{enumerate}
\item[] until $ o'' = o' $.
\end{algo}

This algorithm only needs space to store two alternatives at any iteration, and checks sanctioning w.r.t. one rule at every iteration. Repeated applications of this algorithm can answer \pb{$R$-comparison} queries for $ R ∈ \{≻,∼,⋈\} $; for instance, to check if $ o ⋈_φ o' $, we check that $ o ≽_φ o' $ does not hold and that $ o' ≽_φ o $ does not hold either. \footnote{Recall that  \NPSPACE\ = \cclass{co-NPSPACE} = \PSPACE\ . }

\end{commentaire}

For §CP§, tractability of comparisons, except in some trivial cases, comes at a heavy price in terms of expressiveness: \pb{$≽$-comparison} is tractable for CP-nets when the dependency graph is a polytree \citep[Theorem 14]{Boutilieretal:jair04}, but \cite[Theorems 15, 16]{Boutilieretal:jair04} prove that \pb{$≽$-comparison} is already \NP-hard for the quite restrictive language of binary-valued, directed-path singly connected CP-nets, which are acyclic. \cite[Prop. 7, Corollary 1]{Goldsmithetal:jair08} prove that \pb{$
≽$-comparison},  \pb{$≻$-comparison},  \pb{$⋈$-comparison} and  \pb{$∼$-comparison} are \PSPACE-complete for $ §1-CP⋫∧§ $ and for consistent, locally complete formulas of $ §1-CP⋫§ $. More precise hardness results for acyclic CP-nets are also shown in \cite{LukasiewiczMalizia:art-int19}. Proposition~\ref{prop:compl-comparisons} completes the picture.

\begin{TODO}
Generalize next prop : \pb{$∼$-comparison} is easy for $ §k-CP §\mathord{ \acycllex} $.
\end{TODO}

\begin{propositionrep}\label{prop:compl-comparisons}
\pb{$≻$-comparison} and \pb{$ ⋈ $-comparison} are \NP-hard for the language of acyclic CP-nets, and tractable for polytree CP-nets.
\ifklptcompat
\pb{$∼$-comparison} is easy for $ §1-CP §\mathord{ \acycllex} $.
\fi
\end{propositionrep}

\begin{toappendix}
Proposition~\ref{prop:compl-comparisons} is proved using
a result about the \pb{ordering} query introduced in~\cite{Boutilieretal:jair04}: it is a particular case of the \pb{top-$p$} that is recalled in section~\ref{sec:top-p}.

\defquery {Ordering}  Given $ S ⊆ \dom{\cal X } $ with $ \card S = 2 $, and $ φ $, return some $ o ∈ S $ such that $ o' ⊁_φ o $, where $ o' $ is the other element of $ S $.!


Note that when $ S $ contains exactly two elements, at least one of them is not strictly dominated by the other; 
it the two elements in $ S $ are incomparable, then the \pb{ordering} query may return any one of them.

\begin{proof} 
Note that for ayclic CP-nets (and thus for polytree CP-nets), \pb{$ ≽ $-comparison} and \pb{$≻$-comparison} are ``almost" equivalent, in the sense that for different alternatives $ o $ and $ o' $, $ o ≻_φ o' $ iff $ o ≽_φ o' $ (because acyclic CP-nets are consistent). In particular, \pb{$ ≽ $-comparison} can be reduced to \pb{$≻$-comparison} for consistent languages, thus \pb{$≻$-comparison}  is \NP\ hard for acyclic CP-nets because \pb{$ ≽ $-comparison} is hard for this language \cite[Theorems 15, 16]{Boutilieretal:jair04}.

\pb{$≻$-comparison} can also be reduced, still for languages that guarantee consistency, to \pb{$⋈$-comparison}: consider alternatives $ o ≠ o' $, in order to check if $ o ≻ o' $ we can ask if $ o ⋈ o' $: if the answer is ``yes'', then $ o ⊁ o' $; if the answer is ``no'', ask the \pb{ordering} query for $ S = \{o,o'\} $: the answer must be, in polynomial time for acyclic CP-nets \citep[Theorem 5]{Boutilieretal:jair04}, that $ o ⊁ o' $ or $ o' ⊁ o $: if the answer is that $ o ⊁ o' $, it answers the initial query; if the answer is that $ o' ⊁ o $, since we know that $ o ≻ o' $ or $ o' ≻ o $ because $ o $ and $ o' $ are not incomparable and $ ≽ $ is antisymmetric, it must be the case that $ o ≻ o' $.

Finally, \pb{$ ≽ $-comparison} is tractable for polytree CP-nets, and two calls of this query at most can answer \pb{$≻$-comparison} and \pb{$⋈$-comparison}.

\end{proof}
\end{toappendix}

\begin{TODO}
Can we generalize, at least for §1-CP⋫§? It seems difficult with importance statements, because the results follow from work on planning with unary operators by \cite{BrafmanDomshlak:jair03}, but if there are ``free'' attributes then it does not seem to be a unary operator anymore?
\end{TODO}

\begin{commentaire}\mdfsubtitle{Si on parle de consistence ?}
Theorem 2 in \cite{Goldsmithetal:jair08}: completeness even still holds for locally consistent, locally complete and ``complete'' / antisymmetric §1-CP⋫§\ theories (but the reduction yields a non-conjunctive theory.)
\end{commentaire}

\subsection{Comparing theories}

Checking if two theories yield the same preorder can be useful during the compilation process. We say that two formulas $ φ $ and $ φ' $ are equivalent if they represent the same preorder, that is, if $ ≽_φ $ and $ ≽_{φ'} $ are identical; we then write $ φ ≡ φ' $.

\defquery {Equivalence} Given two formulas $ φ $ and $ φ' $, are they equivalent?!


Consider a formula $ φ ∈ §CP§ $, two alternatives $ o, o' $, and let $ φ' = φ ∪ \{o ≥ o'\} $: clearly $ o ≽_{φ'} o' $, thus $ φ ≡ φ' $ if and only if $ o ≽_φ o' $. Therefore, if a language~$ \cal L ⊆ §CP§ $ is such that adding the CP statement $ o ≥ o' $ to any of its formulas yields a formula that is still in $ \cal L  $, then \pb{equivalence} has to be at least as hard as \pb{$≽$-comparison} for $ \cal L  $.  This is the case of §CP§. The problem remains hard for §1-CP⋫§, because it is hard to check the equivalence, in propositional logic, of the conditions of statements that entail a particular swap $ x ≥ x' $.


\begin{example} \label{exple:eqiv=hard}
Consider three attributes $ A $, $ B $ and  $ C $ with respective domains $ \{a, \bar a\} $, $ \{b, \bar b\} $ and $ \{c₁, c₂, c₃\} $. Consider two CP statements $ s = \bar a: c₁ ≥ c₂ $ and $ s' =  b: c₂ ≥ c₃ $, and let $ φ = \{s, s', a : c₁ ≥ c₃\} $. Because of statements $ s $ and $ s' $ we have $  \bar abc₁ ≥_φ \bar abc₂ ≥_φ \bar abc₃ $; also, $  abc₁ ≥_φ abc₃ $ because of statement $ a : c₁ ≥ c₃ $. Hence, for any $ u ∈ \dom A × \dom B $, if $ u ⊧ a ∨ (\bar a b) $ then $ uc₁ ≥ uc₃ $. Thus $ φ ≡ φ ∪ \{\bar a b : c₁ ≥ c₃\}  ≡ φ ∪ \{b : c₁ ≥ c₃\} $: the last equivalence follows from the fact that $ a ∨ (\bar a b) ≡ a ∨ b $.
\end{example}


\begin{propositionrep}
\pb{equivalence} is \coNP-hard for $ §1-CP⋫\!∧§\! \acycl $, and for $ §1-LPT§∧ $, both restricted to binary attributes.
\end{propositionrep}

\begin{proofsketch}
It is not difficult to reduce \pb{equivalence} in propositional logic to \pb{equivalence} for $ §1-CP⋫∧§\acycl $, and for $ §1-LPT§∧ $.
\end{proofsketch}

\begin{toappendix}
Given a propositional language $ \cal P  $ we define $ \cal P ^{∨} $ to be the set of finite disjunctions of formulas in $ \cal P  $, and:
\begin{langlist}
\item $ §1-CP§ ⋫\cal P  $ is §1-CP⋫§\ restricted to those statements such that the condition is in $ \cal P  $
\item $ §1-LPT§\cal P  $ is §1-LPT§\ restricted to those LP-trees such that the condition of every rule is in $ \cal P  $.
\end{langlist}

The proof of the proposition is based on the following lemma, which formalizes the intuition suggested by Example~\ref{exple:eqiv=hard}.

\begin{lemma}
Given a propositional language $ \cal P  $ closed for conjunction, \pb{equivalence} for $ \cal P ^{∨} $ (in the sense of propositional logic), 
reduces to \pb{equivalence} for $§1-CP⋫§\cal P $\ restricted to acyclic
formulas, and to \pb{equivalence} for $ §1-LPT§\cal P $.
\end{lemma}

\begin{proof}
Consider two formulas $ α = ⋁_{i∈I} αᵢ $ and $ α' = ⋁_{i∈I'}α'ᵢ $ over a set $ \cal X  $ of binary attributes, with all $ αᵢ $'s and $ α'ᵢ $'s in $ \cal P  $; take some binary attribute $ X ∉ \cal X  $, with values $ x $ and $ \bar x $, and let $ φ = \{αᵢ: x ≥\bar x | i ∈ I\} $ and $ φ' = \{ α'ᵢ: x ≥\bar x | i ∈ I'\} $. Note that $ φ, φ' ∈ §1-CP⋫§\cal P  $, that they are acyclic, and that they can be computed in time polynomial in $ |α|+|α'| $.
Then $ φ^* = \{(ox,o\bar x) | o ∈ \dom{\cal X }, o ⊧ α \} $ and for every $ o₁, o₂ ∈ \dom{\cal X } $, for every $ x₁, x₂ ∈ \dom X $, $ o₁x₁ ≽_φ o₂x₂ $ if and only if $ o₁ = o₂ $, $ x₁ = x $, $ x₂ = \bar x $ and $ o₁ ⊧ α $; similarly, $ o₁x₁ ≽_{φ'} o₂x₂ $ if and only if $ o₁ = o₂ $, $ x₁ = x $, $ x₂ = \bar x $ and $ o₁ ⊧ α' $. Thus $ α ≡ α' $ if and only if for every $ o ∈ \dom{\cal X } $, $ o ⊧ α ⇔ o ⊧ α' $, iff for every $ o ∈ \dom{\cal X } $, $ ox ≽_φ o\bar x ⇔ ox ≽_{φ'} o\bar x $, if and only if $ φ ≡ φ' $.

Similarly, we can define two linear 1-LP-trees $ ψ $ and $ ψ' $ as follows: the top $ | \cal X | $ nodes are labelled with attributes from $ \cal X $, in any order and with no rule; then there is one node labelled with $ X $, and the same preference rules as above.
\end{proof}
\end{toappendix}

As usual, comparing two formulas is easier for languages where there exists a canonical form. This is the case of acyclic CP-nets, as shown by \cite[Lemma 2]{KoricheZanuttini:aij10}; their proof makes it clear that the canonical form of any acyclic CP-net $ φ $ can be computed in polynomial time. Hence:

\begin{proposition}
\pb{Equivalence} is in \P\ for §CP-net§.
\end{proposition}

\begin{TODO}
Check \pb{equivalence} for LP trees!
\end{TODO}


\subsection{Top $ p $ alternatives} \label{sec:top-p}

Given a set of alternatives $ S $ and some integer $ p $,  we may be interested in finding a subset $ S' $ of $ S $ that contains $ p $ ``best'' alternatives of $ S $, in the sense that for every $ o ∈ S' $, for every $ o' ∈ S ∖ S' $ it is not the case that $ o' ≻_φ o $. Note that such a set must exist, because $ ≻_φ $ is acyclic. The \pb{Top-$ p $} query is usually defined for totally ordered sets; a definition suited to partial orders is given in  \cite{Wilson:aij11} (where it is called \textit{ordering}), we adopt this definition here:

\defquery {Top-$ p $}  Given $ S ⊆ \dom{\cal X } $, $ p < |S| $,  and $ φ $, find $ o₁, o₂, …, o_p ∈S $ such that for every $ i ∈ {1,…,p} $, for every $ o' ∈ S $, if $ o' ≻_φ oᵢ $ then $ o' ∈ \{o₁,…,o_{i-1}\} $.!


Note that if $ o₁, o₂, …, o_p $ is the answer to such query, if $ 1 ≤ i < j ≤ p $, then it can be the case that $ oᵢ ⋈ o_j $, but it is guaranteed that $ o_j ⊁ oᵢ $: in the context of a recommender system for instance, where one would expect alternatives to be presented in order of non-increasing preference, $ oᵢ $ could be safely presented before $ o_j $.

\cite{Boutilieretal:jair04}  prove that \pb{top-$p$} is tractable for acyclic CP-nets for the specific case where $ |S|=2 $. More generally, \pb{$ ≻ $-comparison} queries can be used to compute an answer to a \pb{top-$p$} query (by asking \pb{$ ≻ $-comparison} queries for every pair of elements of $ S $, the number of such pairs being in $ ϴ(|S|²) $). \nouv{Thus \pb{top-$p$} is tractable for every language where \pb{$ ≻ $-comparison} is tractable; this is the case in particular of §GAI§\ and §LPT§.}

\ifklptcompat
However, \cite{Wilson:aij11} shows that an upper approximation of $ ≻ $ is sufficient, and proves that such an approximation can be obtained in time polynomial in $ |φ| $ for a restricted class of lexico-compatible formulas of §1-CP∧§\ that contains $§1-CP§\acycl$~\citep[Th. 5]{Wilson:aij11}. 

We prove that this result does indeed hold for the full class of lexico-compatible formulas of  §1-CP∧§.\stefan{erm, c'est où, ce résultat ?}

\begin{TODO}
Does it hold for §1-CP§ ???
\end{TODO}

\begin{propositionrep} \label{prop:ordering-easy}
\pb{top-$p$} can be answered in time which is polynomial in the size of $ φ $ and the size of $ S $ for k-lexico-compatible formulas (for fixed $ k $); and for §LPT§.
\end{propositionrep}

\begin{appendixproof}
It suffices to prove the result for the case where $ |S|=2 $.
The fact that the query is tractable for §LPT§\ in this case is a simple consequence of the fact that \pb{$≻$-comparison} is tractable for §LPT§.

For $k$-lexico-compatible formulas, given some $ φ ∈ §CP§ $ known to be $k$-lexico-compatible, and given a pair of alternatives $ \{o,o'\} $, we can run the algorithm proposed in Section \ref{sec:lex-comp} that builds a complete LP tree $ ψ $ that extends $ φ $, but build one branch only, the one that corresponds to the pair $ \{o,o'\} $, as long as the chosen attributes have equal values for $ o $ and $ o' $: when reaching a node where the chosen set of attributes $ T $ si such that $ o[T] ≠ o'[T] $, the node decides the pair, and $ o $ and $ o' $ can be ordered accordingly, as in the case of LP trees: if $ o[T] > o'[T] $, then $ o ≻_ψ o' $, thus $ o' ⋡_ψ o $, hence $ o' ⋡_φ o $; similarly, if $ o'[T] > o[T] $, then it cannot be the case that $ o ≻_φ o' $. The algorithm will not return §FAILURE§\ because $ φ $ is known to be $k$-lexico-compatible. The branch has no more that $ | \cal X | $ nodes, and there is, for fixed $ k $, a polynomial number of possible labels to try at each node.
\end{appendixproof}
\fi

\begin{notAAMAS}
We do not consider here the case were $ S $ is given an intensional definition, as the set of solutions of some CSP over $ \dom{\cal X } $ for instance; this particular setting has been studied by e.g.~\cite{Brafmanetal:kr10,FreuderHeffernanWallaceWilson:constraints10}.
\end{notAAMAS}

\subsection{Optimization}
Instead of ordering a given set, we may want to find a globally optimal alternative. 
We say that alternative $o $ is \emph{undominated} if there is no $  o' \!\!∈\!\! \dom{\cal X }  $ such that $ o' \!\!≻_φ\!\! o  $.
\footnote{\cite{Goldsmithetal:jair08} say that $ o $ is in this case ``weakly undominated''. They also say that $o $ is:
\emph{undominated} if there is no $  o' ∈ \dom{\cal X }  $, $ o' ≠ o $,  such that $ o' ≽_φ o $; \emph{dominating} if for every $  o' ∈ \dom{\cal X }  $, $ o ≽_φ o' $; \emph{strongly dominating} if for every $  o' ∈ \dom{\cal X }  $ with $ o' ≠ o $, $ o ≻_φ o' $.
The complexity of queries related to the latter three definitions is studied in~\cite{FargierMengin:AAMAS21}.}

%

Note that any finite set of alternatives always has at least one undominated alternative.
%
We will consider the following queries:

\defquery{undominated extract} Given $ φ $, return an undominated alternative.!

\defquery{undominated check} Given $ φ $ and an alternative $ o $, is $ o $ undominated?!








These queries are easily shown to be tractable for §LPT§. 






For §CP§-nets, \cite{Boutilieretal:jair04} give a polytime algorithm that computes the only undominated alternative when the dependency graph is acyclic.
\ifklptcompat
For a given formula $ φ $ known to be lexico-compatible, we can run the algorithm proposed in Section \ref{sec:lex-comp} that builds a complete LP tree $ ψ $ that extends $ φ $, but build one branch only: the one that contains the most preferred alternative $ o $; and return that most preferred alternative. Then for every other alternative $ o' $, $ o ≻_ψ o' $, so $ o' \not ≽_φ o $ and $ o' \not ≻_φ o $: $ o $ is undominated w.r.t. φ.
\fi

\cite{Goldsmithetal:jair08} prove that  \pb{undominated check} is \PSPACE-complete for  $ §1-CP⋫§ $, and their reductions for proving hardness
indeed yield formulas of $ §1-CP⋫∧§ $.

For §GAI§$_1$, extracting an undominated alternative can be performed by separately maximizing the unary utilities; and checking if a given alternative is undominated can be done by comparing its utility to that of an extracted undominated alternative.
Undominated extract and undominated check are both NP-hard for §GAI§$_2$ and thus for $§GAI§$ and $§GAI§_k$ in the general case. We will see these results in the next subsection where we make several similar constructions (Proposition \ref{prop:cutextractiongai}). 

\subsection{Cuts}\stefan{we still have to clean up results on GAI that are for queries we are not considering anymore}
\emph{Cuts} are sets of alternatives that are at the same ``level'' with respect to $ ≽ $. For rankings defined with real-valued functions, cuts are defined with respect to possible real values. In the case of pre-orders, we define cuts with respect to some alternative $ o $: given $ φ ∈ §CP§ $, for any $ R ∈ \{≻,≽\} $, for every alternative $ o $, 
we define
\[
 §CUT§^{R,o}(φ) = \{ o' ∈ \dom{\cal X } | o' ≠ o, o' R_φ o \} .
\]

Following \cite{FargierMarquisNiveauSchmidt:aaai14}, we define two families of queries:

\defquery {$R$-cut extraction} Given $ φ, o $, return an element of $ §CUT§^{R,o}(φ) $ (or that it is empty)!
\defquery {$R$-cut counting} Given $ φ, o $, count the elements of $ §CUT§^{R,o}(φ) $!

Note that 

\begin{propositionrep}
 \pb{$≽$-cut counting} and \pb{$≻$-cut counting} are $\sP$-hard for §CP§-nets and acyclic §CP§-nets.\stefan{also for §CP§\acycllexk ?}
\end{propositionrep}
\begin{appendixproof}
 Remember that a vertex cover in a graph $G=(V,E)$ is a set $S\subseteq V$ such that for each edge $uv\in E$ we have $u\in S$ or $v\in S$. The problem $\mathsf{\#VertexCover}$ is, given a graph $G$, to count its vertex covers. $\mathsf{\#VertexCover}$ is well-known to be $\sP$-hard~\cite{Valiant79}, so we will use a reduction from $\mathsf{\#VertexCover}$ to $≻$-\textsc{cut-counting} to establish the claim.
 
 So let $G=(V,E)$ be a graph. For every vertex $v\in V$ we introduce an attribute $V_v$ and for every edge $e= uv\in E$ we introduce an attribute $E_{uv}$. Note that for convenience we denote $E_{uv}$ also by $E_e$ sometimes. Finally, we introduce attributes $D_i$ for $i\in [|V|+|E|+1]$. The attributes $V_v$ have no parents. Let $e_1, \ldots, e_m$ be an order of the edges in $E$ where $e_i = u_iv_i$. For $i>1$ the attribute $E_{e_i}$ has the parents $V_{u_i}, V_{v_i}, E_{e_{i-1}}$. The attribute $E_{e_1}$ has parents $V_{u_1}, V_{v_1}$. Finally, the attributes $D_j$ all have the single parent $E_{e_m}$.
 
 We next describe the CPTs for all attributes: all attributes have values in $\{0,1\}$. All $V_{v}$ have the order $1 \ge 0$. For all $D_i$, we have that if $E_{e_m}$ has value $0$ then the order is $0 \ge 1$ and if $E_{e_m}$ has value $1$, then $1\ge 0$. For $E_{e_1}$ we have the order $1\ge 0$ if and only if at least one of $V_{u_1}, V_{v_1}$ has value $1$ and the order $0\ge 1$ otherwise. Finally, for $i>1$,  we have the order $1\ge 0$ if and only if $E_{i-1}$ has value $1$ and at least one of $V_{u_i}, V_{v_i}$ has value $1$. Otherwise $E_{e_i}$ has the order $0\ge 1$. Call the resulting CP-net $\phi$.
 
 Note that one can easily see that no attribute in an increasing flipping sequence can ever be flipped back to $0$ from $0$: for the attributes $V_v$ this is immediate. For the $E_{e_j}$ it follows with an easy induction and the fact that it is true for the $V_v$. For the $D_j$ finally it follows from the fact that $E_{e_m}$ can never flip back to $0$.
 
 Let $o$ be the assignment that assigns $0$ to all attributes. Let $o'$ be an assignment such that $o'$ is reachable from $o$ by an increasing flipping sequence, or equivalently $o' \succeq_\phi o$. We claim that if $E_{e_m}$ has value $1$ in $o'$, then $S := \{v \in V\mid o'(V_v) = 1\}$ is a vertex cover of $G$. To see this, first observe that in fact all $E_{e_j}$ must take the value $1$ in $o'$: to flip $E_{e_j}$ to $1$, we must have flipped $E_{j-1}$ before (if it exists) and since we can never flip back to $0$, $E_{j-1}$ must take $1$ in $o'$. But then when we flipped $E_{e_j}$ to $1$, at least one of $V_{v_j}, V_{u_j}$ must have had value $1$ and since we cannot flip it back, in $o'$ one of $V_{v_j}, V_{u_j}$ must have value $1$. So for every $e_j$ we have that one of $V_{v_j}, V_{u_j}$ must have value $1$ which proves that $S$ is a vertex cover as claimed.
 
 Now for $S\subseteq V$, define $o_S$ to be the assignment that assigns $1$ to $V_v$ if and only if $v\in S$, assigns $1$ to all $E_{e_i}$ and assigns $0$ to all $D_j$. We claim that $o_S \succeq_\phi o$ if and only if $S$ is a vertex cover of $G$. First note that if $S$ is a vertex cover, we can flip all $V_v$ accordingly and then iteratively flip all $E_{e_j}$ to reach $o_S$. The other direction is clear from what we saw above, observing that $E_{e_m}$ takes value $1$ in $o'$. 
 
 Observe that for every $o_S$, where $ S $ is a vertex cover, we can flip an arbitrary subset of the $D_j$ to~$1$ to reach an assignment $o'\succeq_\phi o_S \succeq_\phi o$. Note that for different vertex covers $S_1, S_2$, there is no such $o' \succeq o_{S_1}$ and $o' \succeq o_{S_2}$ since $o_{S_1}$ and $o_{S_2}$ differ on the $V_v$ and in the construction of the $o'$ from the $o_S$ we do not change those. It follows that 
 \begin{multline*}
  \{o' \mid o' \succeq_\phi o'', o''(E_{e_m}) = 1\} =
  \\ \bigcup_{S \text{ vertex cover of } G} \{o' \mid o'\succeq_\phi o_S, \forall V_v: o_S(V_v) = o'(V_v)\}
 \end{multline*}
and the union is disjoint. Now for every vertex cover $S$ of $G$, we have 
\begin{align*}
|\{o'' \mid o''\succeq_\phi o_S, \forall V_v: o_S(V_v) = o''(V_v)\}| = 2^{|V|+|E|+1}. 
\end{align*}
 Let $s$ be the number of vertex covers of $G$. It follows that
\begin{align*}
 |\{o' \mid o'\succeq_\phi o\}| &= |\{o' \mid o' \succeq_\phi o, o'(E_{e_m}) = 0\}| + |\{o' \mid o'' \succeq_\phi o, o''(E_{e_m}) = 1\}|\\
 &= |\{o' \mid o' \succeq_\phi o, o'(E_{e_m}) = 0\}| + s 2^{|V|+|E|+1}.
\end{align*}
Now since in no $o' $ with $o'(E_{e_m}) = 0$ any of the $D_j$ can be flipped to $1$ in any increasing flipping sequence, we have 
\begin{align*}
 |\{o' \mid o' \succeq_\phi o, o'(E_{e_m}) = 0\}| < 2^{|V|+|E|},
\end{align*}
since such $o'$ have only $|V|+|E|-1$ attributes with domain $\{0,1\}$ which are not forced to be constant $0$.
Consequently, $s$ can be inferred from $|\{o' \mid o'\succeq_\phi o\}|$ by a single integer division which completes the reduction.

This proves that \pb{≽-cut counting} for acyclic CP-nets is as hard as \pb{\#VertexCover}; this holds for \pb{≻-cut counting} since in the case of acyclic CP-nets, ≽ is antisymmetric. And this hardness result extends to the larger class of CP-nets.
\end{appendixproof}

\begin{propositionrep}
\pb{$≻$-cut counting} is $\sP$-complete for §GAI§, §GAI§$_k$ and §GAI§$_1$.
\end{propositionrep}

\begin{appendixproof}
For containment in $\sP$, observe that all elements in $ §CUT§^{≻,o}(φ) $ have polynomial size, so we can easily guess them and compare in polynomial time to $o$ since \pb{$≻$-comparison} can be solved in polynomial time for §GAI§.
	
For hardness, we reduce from the problem \pb{\#SubsetSum} which is, given a set $S=\{s_1, \ldots, s_n\}$ of positive integers and an additional integer $k$, to count the number of subsets of $S$ that sum up to $k$. \pb{\#SubsetSum} is well-known to be $\sP$-complete, see e.g.~\cite{FaliszewskiH09}. It will be convenient to work with a slight variant which we call \pb{\#SubsetSum$_>$} and which is, given the same type of input as for \pb{\#SubsetSum}, to count the number of subsets of $S$ which sum up to a value greater than $k$. There is an easy oracle reduction from \pb{\#SubsetSum} to \pb{\#SubsetSum$_>$}: given an input $S$, $k$, call an oracle for \pb{\#SubsetSum$_>$} on the two inputs $S, k-1$ and $S, k$. Then the answer to the \pb{\#SubsetSum} instance is the difference of the answers of the oracle calls. It follows that \pb{\#SubsetSum$_>$} is $\sP$-hard.

We now reduce \pb{\#SubsetSum$_>$} to \pb{$≻$-cut counting} for §GAI§$_1$. So let $S= \{s_1, \ldots, s_n\}$ and $k$ be an instance of \pb{\#SubsetSum$_>$}. We construct $n$ functions $g_i(X_i)$ for $i=1, \ldots, n$ where $\dom {X_i} = \{0,1\}$. We set $g_i(0)=0$ and $g_i(1)=s_i$. Moreover, we add a function $g_Y(Y)$ where $\dom{Y}= \{0,1\}$ and $g_Y(0) = 1$ and $g_Y(1) = k$. Set $\varphi= \{g_1, \ldots, g_n, g_Y\}$ and $\cal X = \{X_1, \ldots, X_n, Y\}$. This completes the construction of the GAI. Call the induced relation $\succeq$.

To complete the reduction, let $o^*\in \dom{\cal X}$ be the alternative that sets $Y$ to~$1$ and all other attributes to $0$. Then $g_\phi(o^*) = k $. Moreover, for $o\in \dom{\cal X}$, we have that $g_\phi(o)> k$ if and only if $o(Y) = 1$ and there is an $i\in  [n]$ such that $o(X_i)=1$---i.e.~the set $\{i\in [n]\mid o(X_i)=1\}$ is non-empty---, or $o(Y) = 0$ and $\sum_{i\in [n]} g_i(o[X_i]) = \sum_{i\in [n]\colon o(X_i)=1} s_i > k$. Note that there are $2^n-1$ alternatives of the former type, corresponding to the non-empty subsets of $[n]$, so the number of subsets of $S$ that sum up to values greater than $k$ is $ |§CUT§^{≻, o^*}(φ)| - 2^n+1$. Thus, one oracle call to \pb{$≻$-cut counting} allows solving \pb{\#SubsetSum$_>$} in polynomial time which completes the reduction.
\end{appendixproof}

\begin{propositionrep}
\pb{$≽$-cut extraction} is tractable for §CP§, and \pb{$≻$-cut extraction} is tractable for acyclic CP-nets. \pb{$≻$-cut counting} and \pb{$≻$-cut extraction} are \PSPACE-hard for $ §1-CP⋫∧§ $.
\ifklptcompat
For $ §CP§\acycllexk $, \pb{$≻$-cut extraction} is equivalent to \pb{$≽$-cut extraction} and is tractable.
\fi
\pb{$≻$-cut extraction}, \pb{$≽$-cut extraction} and \pb{$≻$-cut counting} 
are tractable for LP-trees.
\end{propositionrep}

\begin{appendixproof}
\pb{$≽$-cut extraction} is easy for §CP§: given $ o $ and $ φ $, in order to return an element of $ §CUT§^{≽,o}(φ) $, it is sufficient to find one statement in $ φ $ which sanctions an improving swap for $ o $.
For acyclic CP-nets (and more generally for any language that guarantees consistency), $ ≻ $ is the asymmetric part of $ ≽ $, thus \pb{$≻$-cut extraction} is equivalent to \pb{$≽$-cut extraction} and is tractable. 

Note that alternative $o $ is undominated iff $ §CUT§^{≻,o}(φ) = ∅ $, iff $ |§CUT§^{≻,o}(φ)| =  0 $;
therefore, \pb{$≻$-cut counting} and \pb{$≻$-cut extraction} are at least as hard as \pb{undominated check}, they are therefore \PSPACE-hard for $ §1-CP⋫∧§ $.

Finally, \pb{$≻$-cut extraction}, \pb{$≽$-cut extraction} and \pb{$≻$-cut counting} are tractable for LP-trees: for LP-tree $ φ $, given $ o $, in order to find some $ o' $ such that $ o' ≻_φ o $ (resp. $ o' ≽_φ o $), it is possible to traverse the tree, starting at the root, guided by the values assigned by $ o $, until reaching a node where the value(s) assigned by $ o $ for the attributes at that node is/are strictly dominated (resp. dominated) by other values at that node. Also, when going down $ φ $ in the branch that corresponds to $ o $, it is possible, at each node $ N $ encountered, labelled with $ T $, to count the number of $ t' $ in $ \dom T $ such that $ t > o[t'] $ (according to the preference rule $ β:≥^β $ at $ N $ such that $ o ⊨ β $), and to multiply this number by the sizes of the domains of the attributes that have not been encountered yet; adding these sums of products along the branch will give the number of alternatives $ o $ such that $ o' ≻_φ o $.

\end{appendixproof}

\begin{propositionrep}
	\pb{$≽$-cut extraction}, \pb{$\succ$-cut extraction}, \pb{undominated check}, and \pb{undominated  extract} are tractable for §GAI§$_1$.
\end{propositionrep}
\begin{appendixproof}
	Given a GAI $\varphi$, we can simply choose the values for the attributes in such a way that the utilities are maximized. Since the utilities are unary, this leads to a consistent and thus also maximal alternative $o^*$. For \pb{$≽$-cut extraction}, $o^*$ is always a valid output, so it solves the problem independent of the additional input alternative $o$. For \pb{$\succ$-cut extraction}, we check if $g_\varphi(o^*)> g_\varphi(o)$. If so, we return $o^*$ again. Otherwise, due to the maximality of $o^*$, we have $g_\phi(o^*)= g_\phi(o)$ and thus there is no alternative strictly dominating $o$ and thus no valid output.
	
	For \pb{undominated check} we have that $o$ is undominated if and only if $g_{\phi}(o) = g_{\phi}(o^*)$ which we can check efficiently. Finally, for undominated extract we can simply return $o^*$. 
\end{appendixproof}
The problems above become intractable as soon as we allow utility functions of arity at least $2$.

\begin{propositionrep}\label{prop:cutextractiongai}
	\pb{$≽$-cut extraction} and \pb{$\succ$-cut extraction} are $\NP$-complete for §GAI§$_k$ for $k\ge 2$ and~§GAI§. \pb{undominated check} is $\coNP$-complete and \pb{undominated extract} is \NP-hard for §GAI§$_k$ for $k\ge 2$ and for~§GAI§.
\end{propositionrep}

\begin{appendixproof}
	Containment in \NP, resp.~\coNP, is easy to see in all cases since alternatives can be compared efficiently
	
	We show hardness for all problems by reduction from \pb{$3$-Coloring} which is, given a simple, undirected graph $G=(V,E)$, to decide if there is an assignment $c:V\rightarrow \{r,g,b\}$ such that for all edges $uv\in E$ we have $c(v)\ne c(u)$. The mapping $c$ is called a \emph{coloring} and it is said to be valid if it satisfies the condition on the edges. \pb{$3$-Coloring} is well-known to be \NP-complete, see e.g.~\cite[Theorem 9.8]{Papadimitriou94}.

	We use the same construction of a §GAI§$_2$ $\varphi_G$ from a graph $G$ for all problems.
	So let a graph $G$ be given in which w.l.o.g.~every vertex has at least two neighbors (vertices with fewer than two neighbors can iteratively be deleted without changing the answer to the \pb{$3$-Coloring} question). We also assume that $G$ is connected; if it is not, we can connect the different connected components iteratively by adding edges without changing the answer to the \pb{$3$-Coloring} question. We construct a §GAI§$_2$ representation as follows: for every vertex $v\in V$, we introduce an attribute $X_v$ with domain $\dom{X_v} = \{r,g,b,d\}$. For every edge $e=uv$, we construct a utility function $g_{uv}$ in the variables $X_u, X_v$ and which takes value $1$ on inputs $rb, rg, br, bg, gr, gb, dd$ and $0$ on all other inputs. Setting $\varphi_G = \{g_{uv}\mid uv\in E\}$ completes the construction of the GAI $\varphi_G$. Let $\succeq$ be the order that $\varphi_G$ induces.
	
	We first show hardness for \pb{$≽$-cut extraction}. To this end, let $o_d$ be the alternative in which all attributes take value $d$. Then all $g_{uv}$ evaluate to $1$ on $o_d$, so $g_{\varphi_G}(o_d) = |E|$. Now consider $o \in §CUT§^{\succeq,o_d}$. Assume first that some attribute of $o$ takes value $d$. Since not all attributes can take value $d$ and $G$ is connected, there must be an edge $uv$ such that $d = o(u) \ne o(v)$. Then $g_{uv}(o)=0$ and $g_{\varphi_G}(o)< |E| = g_{\varphi_G}(o_d)$, so $o \notin §CUT§^{\succeq,o_d}$ which contradicts the choice of $o$. Consequently, we must have that all $X_v$ takes values in $\{r,g,b\}$ in $o$. Moreover, for all $uv\in E$, we must have that $o(X_u)\ne o(X_v)$. Thus, setting $c(v) = o(X_v)$ for all $v\in V$ yields a valid coloring of $G$. So if there is an element in $§CUT§^{\succeq,o_d}$, the graph $G$ is $3$-colorable. The other way round, if $G$ has a valid $3$-coloring $c$, then defining $o$ for all $X_v$ by $o(X_v) = c(v)$ yields an alternative in $§CUT§^{\succeq,o_d}$. This shows \NP-hardness of \pb{$≽$-cut extraction} for §GAI§$_2$ and thus for all §GAI§$_k$ with $k\ge 2$ and §GAI§.
	
	The reasoning for \pb{$\succ$-cut extraction} is similar. The only difference is that for one arbitrary edge $uv$ we set $g_{uv}(d,d)$ to $0$. Call the resulting GAI $\varphi_G'$. We have $g_{\varphi_G'}(o_d)= |E|-1$. The rest of the reduction and the argument for completeness is exactly as that for \pb{$≽$-cut extraction}.
	
	For \pb{undominated check}, observe that $o_d$ is dominating for $\varphi_G'$ if and only if there is no alternative $o$ with $g_{\varphi_G'}(o_d) < g_{\varphi_G'}(o)$. Reasoning as above, this is exactly the case if and only if $G$ has no valid $3$-coloring. Thus \pb{undominated check} is $\coNP$-hard. 
	
	Finally, to show hardness of \pb{undominated extract}, multiply all utility values in $g_G$ by $2$. Then, for one arbitrary edge $uv$ set $g_{uv}(d,d)$ to $1$. Call the resulting GAI $\varphi_{G}''$. Then we have $g_{\varphi_G''}(o_d)= 2|E|-1$. Moreover, for all alternatives $o$ encoding a valid $3$-coloring, we have $g_{\varphi_G''}(o)= 2|E|$. Finally, for all other alternatives $o$, we have $g_{\varphi_G''}(o)\le 2|E|-2$. So in any case an undominated alternative is either a valid $3$-coloring of the graph $G$ or $o_d$, hence $o_d$ is undominated if and only if $G$ is not $3$-colorable which shows that \pb{undominated extract} is \NP-hard.
\end{appendixproof}

\section{Transformations}
\label{sec:transfo}

\begin{table}
\newcommand{\entetecol}[1]{\multicolumn1{c}{\begin{rotate}{70} \parbox[t]{4.5em}{#1}\end{rotate}}}
\newcommand{\entetecolh}{\multicolumn{1}{c}{\vrule height 5.4em width0pt}}
\tabcolsep2pt
\newbox\mybox\setbox\mybox=\hbox{\NP}
\newcommand{\mycwd}{\hspace*{\wd\mybox}}
\def\P{✔}\def\NPh{✘\!∘}\def\NPc{✘}\def\PSPc{\ensuremath{✘\!✘}}

\begin{center}
\begin{tabular}{|l|c|c|c||c|c|c|c||c|c|c||c|c|}
\entetecolh  & \entetecol{§GAI§}& \entetecol{§GAI§$_k$}&  \entetecol{§GAI§$_1$}&  \entetecol{§CP§} & \entetecol{§1-CP§} &  \entetecol{§1-CP⋫§}  &  \entetecol{§1-CP⋫∧§}   & \entetecol{§CP§net} & \entetecol{§CP§net\acycl} & \entetecol{§CP§net\acyclpoly}  & \entetecol{§LPT§}  & \entetecol{$ §1-LPT§^{§lin§} $ complete}  \\ \hline
& \mycwd & \mycwd & \mycwd & \mycwd & \mycwd & \mycwd & \mycwd & \mycwd & \mycwd & \mycwd & \mycwd \\[-1em]
	\pb{conditioning} & \P & \P & \P &  & \Out & \Out & \Out & \Out & \Out & \Out  & \P & \P \\ \hline
	\pb{conjunction} & \Out & \Out & \Out &  & \Out  & \Out  & \Out & \Out  & \Out & \Out &    &  \Out \\ \hline
	\pb{disjunction} & \Out & \Out & \Out & \Out & \Out  & \Out  & \Out & \Out  & \Out & \Out & \Out &  \Out \\ \hline
	\pb{Lower projection} & \Out & \Out & \P &  &  \Out & \Out & \Out & \Out & \Out & \Out &  & \P \\ \hline
	\pb{\pb{Weak opt. proj.}} & \NPh & \NPh & \P &  & \Out  & \Out & \Out & \Out & \Out &  &  & \P \\ \hline
	\pb{\pb{Strong opt. proj.}} & \NPh  & \NPh & \P &  &  \Out & \Out & \Out & \Out & \Out &  &  & \P \\ \hline
\end{tabular}
\end{center}

\medskip\raggedright
Each column corresponds to one sublanguage of §CP§.
They are sorted in order of decreasing expressiveness from left to right, except when columns are separated by double lines.
For each transformation and sublanguage:
$ \P = $ polytime answer; 
$\NPh = $ \NP/\coNP-hard transformation;
\Out = transformation result may be outside the language.


\caption{\label{table:transfos} Complexity of Transformations.}
\end{table}

Several transformations have been studied in the literature on knowledge compilation. A transformation takes as input one or more formulas, and, possibly, other arguments like some attributes, and returns another formula. Table~\ref{table:transfos} summarizes our results on these transformations. As can be seen from the table, for many sublanguages of CP and transformations, the result of the transformation may be outside that sublanguage.

\subsection{Conditioning}

Several studies, in particular in the context of propositional logic like e.g. \citep{DarwicheMarquis:jair02} work with a  syntactic definition of this transformation; however, in logic, these definitions have a clear semantic counterpart. In the case of CP statements, we shall see that there are languages for which the transformation cannot always be applied, so we give a semantic description, similar to the one given by \cite{FargierMarquisNiveauSchmidt:aaai14}. 

Given a preference relation $ ≽$ on $\cal X$ and a partial instantiation $ u ∈ \dom U $ for some $ U ⊆ \cal X  $, let $ ≽^{|u} $ be the relation defined for every $ r,r' ∈ \dom{\cal X  - U} $ by $ r ≽^{|u} r' $ if and only if $ ru ≽ r'u $. It is straightforward to check that $ ≽^{|u} $ is a preorder. Conditioning a formula $\phi$ in a language consists in computing a formula of the same language  representing $ ≽_φ^{|u} $.

\defquery{conditioning} Given a language $ \cal L$, a formula $ φ$ of $ \cal L$ and  an instantiation  $ u ∈ \dom U ⊆ \cal X $, compute  a formula $ φ' ∈ \cal L $ that represents $ ≽_φ^{|u} $.!


For §LPT§, a simple syntactic transformation on a formula $ φ $ allows, for every attribute $ X $ and every value $ x ∈ \dom X $, to represent $ ≽_φ^{|x} $: for every node $ N $, whose label contains $ X $, remove $ X $ from the label of $ N $, remove the node if it contains no other attribute; if $ N $ has several children, keep only those that correspond to instantiation $ X=x $ (there will only be one if the label of $ N $ contains no other attribute); at the nodes below $ N $, replace every rule $ α : ≥ $ by $ α^{|x} : ≥ $, where $ α^{|x} $ is the result of conditioning applied to $ α $, as defined by e.g. \cite{DarwicheMarquis:jair02}; remove the rule if $ α^{|x} ⊧ ⊥ $; otherwise, since we assume that $ ≥ $ is given in extension, it is easy to keep only the pairs $ (u,u') $ such that $ u[X] = u'[X]=x $ and remove $ x $ from them. This can be performed with a single traversal of the tree.

Even simpler, conditioning a $GAI$ with $ X=x $ amounts to removing from every sub-utility that bears on $ X $ the cases where $ X ≠ x $. 

The next example is an acyclic CP-net, whose dependency graph is even linear, for which there is a conditioning transformation, the result of which cannot be expressed in §1-CP⋫§.

\begin{example}
Consider 3 binary attributes $ A, B, C $, with respective domains $ \{a, \bar a\} $, $ \{b, \bar b\} $, $ \{c, \bar c\} $, and let $$ φ = \{a≥\bar a, a:b≥\bar b, \bar a:\bar b≥b, b:c≥\bar c, \bar b: \bar c ≥ c \}. $$ The underlying, acyclic dependency graph has set of edges $ \{ (A,B), \allowbreak (B,C)\} $. Then $ abc ≽_φ ab\bar c ≽_φ a\bar b\bar c ≽_φ \bar a\bar b\bar c ≽_φ \bar a\bar bc ≽_φ \bar abc ≽_φ \bar ab\bar c $, thus $ abc ≽_φ ab\bar c ≽_φ \bar abc ≽_φ \bar ab\bar c $, that is: $ ac ≽_φ^{|b} a\bar c ≽_φ^{|b} \bar ac ≽_φ^{|b} \bar a\bar c $. However, $ ≽_φ^{|b} $ cannot be represented in §1-CP⋫§.
\end{example}

Note that, in the example above, $ ≽_φ^{|b} $ can be represented in §1-CP§, with formula $ \{|C:a≥\bar a, c ≥\bar c,  \} $. The next example is another CP-net, with a cycle in the dependency graph, for which there is a conditioning transformation, the result of which cannot be expressed in §1-CP§.

\begin{example}
Consider 3 binary attributes $ A, B, C $, with respective domains $ \{a, \bar a\} $, $ \{b, \bar b\} $, $ \{c, \bar c\} $, and let $$ φ = \{b\bar c: \bar a≥ a, ¬(b\bar c): a≥\bar a, c:b≥\bar b, \bar c: \bar b≥b, a:c≥\bar c, \bar a: \bar c ≥ c \}. $$ The underlying dependency graph has set of edges $ \{ (B,A), (C,A), \allowbreak(C,B), (A,C )\} $, it is not acyclic. $ φ $ represents the preorder that is the transitive closure of $ abc ≽_φ a\bar bc ≽_φ a\bar b\bar c≽_φ \bar a\bar b\bar c ≽_φ \bar ab\bar c ≽_φ\bar abc ≽_φ \bar a \bar bc $ and $ \bar a b\bar c ≽_φ ab\bar c $, thus $ ac ≽_φ^{|b} \bar a\bar c ≽_φ^{|b} \bar ac $ and $ \bar a \bar c ≽_φ^{|b} a\bar c $: $ φ^* $ contains all swaps sanctioned by the §1-CP⋫§ statements $ c:a≥\bar a $ (because $ ac ≽_φ^{|b}  \bar ac $),  $ \bar c: \bar a≥a $, $ a: c ≥ \bar c $ and $ \bar a: \bar c ≥ c $, but these statements do not entail that $ ac $ is preferred over $ \bar a\bar c $.
\end{example}

Since §CP§\ can represent any preorder, the result of a conditioning transformation can be expressed in §CP§.
\begin{TODO}
Complexity of conditioning in §CP§. It is probably hard already if the source formula is a CP-net? What about acyclic CP theories?
\end{TODO}

\subsection{Conjunction}

Conjunction is classical for Boolean functions: given two Boolean functions $f^{\cal L}_\varphi$ and $f^{\cal L}_\psi$ represented in a language $ \cal L$, one looks for an $ \cal L $ representation of $f^{\cal L}_\varphi \wedge f^{\cal L}_\psi$.  An analogous definition is also possible when considering formulas representing preferences:

\defquery{Conjunction} Given a language $ \cal L$, two formulas $\varphi$ and  $\psi$ of $ \cal L$  compute  a formula $ \chi$  of $ \cal L$ such that $ o ≽_\chi  o' $ if and only if $ o ≽_\varphi o'$ and $ o ≽_\psi o'$.!

This definition corresponds to  the classical unanimity rule used in ordinal aggregation.\stefan{do we have a reference for that?}

The conjunction of two preorders is a preorder, thus §CP§\ is closed under conjunction. Furthermore, the conjunction of two antisymmetric preorders is antisymmetric too, thus §LPT§\ is closed under conjunction


The §GAI§\ language is not complete for such a transformation: the expressiveness of this languages is limited to complete relations whereas the conjunction of two complete preference relations is not complete in the general case. Consider for instance a GAI decomposition φ such that there at least two alternatives $ o $ and $ o' $ with $ o ≻_φ o' $, and let ψ be the GAI decomposition defined by $ g_ψ(o) = - g_φ(o) $ for every alternative $ o $: then $ o' ≻_ψ o $ and $ o' ⋈_χ o $, where χ denotes the conjunction of φ and ψ. Note that this also applies if $ φ ∈ §GAI§_1 $, therefore neither §GAI§\ nor $ §GAI§_1 $\ are closed under conjunction.



The next example shows that the languages
§CPnet§\acyclpoly, §CPnet§\acycl\ and §CPnet§\ are not closed under conjunction.

\begin{example}
Consider the following two CP-nets in variables $A,B$: $\varphi$ with statements $a\ge \bar a$, $a: b ≥ \bar b$ and $\bar a: \bar b ≥ b$ and $\psi$ with statements $\bar a\ge a$, $a: b ≥ \bar b$ and $\bar a: \bar b \ge b$. The directed graph over the variables in both cases is a polytree that has $A$ as the parent of $B$, so both $\varphi$ and $\psi$ are indeed polytree CP-nets.\istefan{I think these are evern §CP§net\acyclpoly, but I am not sure I am parsing that definition correctly. Can someone check and add this here and in the table?} The complete orders induced by $\varphi$ and $\psi$ are respectively
\[ab \succeq_\varphi a\bar b \succeq_\varphi \bar a \bar b\succeq_\varphi \bar ab,\]
\[\bar a\bar b \succeq_\psi \bar a b \succeq_\psi ab\succeq_\psi  a\bar b.\]
Then the only preferences in $\succeq_{\varphi\land \psi}$ are $ab \succeq_{\varphi\land \psi} a\bar b$ and $\bar a \bar b\succeq_{\varphi\land \psi} \bar a b$. This cannot be expressed by a CP-net, since any CP-net on $ \{A,B\} $ orders the four unary swaps, in particular any CP-net must order for instance $ \{\bar a \bar b , a \bar b\} $, which $ ≽_{φ∧ψ} $ does not. It cannot be represented with an LP-tree nor with a utility since it is not a complete relation.
\jerome{J'ai changé la fin de l'exemple, car un CP-net n'a pas forcément de racine, puisque le graphe n'est pas nécessairement acyclique.}
\end{example}

We now give an example that shows that §1-CP§, §1-CP⋫§, and §1-CP⋫∧§~are not closed under conjunction.
\begin{example}\label{ex:1cpconjunction}
Consider the following two sets of CP-statements over binary attributes $A$ and $B$: $\varphi = \{b:\bar a > a, a: b > \bar b, \bar b: a>\bar a\}$ and $\psi = \{a:\bar b > b, b: a > \bar a, \bar a: b>\bar b\}$. Both $\varphi$ and $\psi$ are in §1-CP⋫∧§. The two sets respectively induce the orders
\[\bar ab \succeq_\varphi ab \succeq_\varphi a\bar b\succeq_\varphi\bar a\bar b,\]
\[a\bar b \succeq_\psi ab \succeq_\psi \bar a b\succeq_\psi\bar a\bar b.\]
For the conjunction, we get that $ab\succeq_{\varphi\land \psi} \bar a \bar b$ and $ab$ is incomparable to the other two alternatives. Thus, this conjunction cannot be expressed by a §1-CP§, since we cannot go from $ab$ to $\bar a\bar b$ in a §1-CP§\ without any intermediate flips.
\end{example}

Many rules of ordinal aggregation could be considered and  this opens a large stream of research which is out of the scope of the present paper - e.g. scoring rules like Borda's, for which the GAI framework is obviously a good candidate language. LP trees on the other hand will probably fail to handle such rules, because the aggregation of several lexicographic orders is generally not a lexiocographic order. The CP   language in itself as such is neither powerful enough to encompass most of the rules, but extensions have been proposed that typically address this question \cite{rossi2004mcp}.\istefan{I wonder if the paragraph above should not go somewhere else? into the conclusion?}

\subsection{Disjunction}
We can define the disjunction operation by symmetry:

\defquery{Disjunction} Given a language $ \cal L$ and two formulas $\phi$ and  $\psi$ of $ \cal L$,  compute  a formula $ \chi'$  of $ \cal L $ such that $ o ≽_{\chi'}  o' $ if and only if $ o ≽_{\phi} o'$ or  $ o ≽_{\psi} o'$. !

Such a definition is nevertheless not really significant in the domain of preference handling\stefan{then why bother defining it? Just delete this?}, since the disjunction of two transitive relations is generally not transitive: it can happen that $o ≻_\phi o'$, $o' ≻_\psi o"$ but neither  $o ≻_\phi o"$ nor  $o ≻_\psi o"$.

\begin{example}\label{ex:disjunction}  $\phi$ is the linear 1-LPT where $B$ is more important than $A$  with $a $ preferred to $\bar a$ and $\bar b$ preferred to $b$; $\psi$ is the linear 1-LPT where $A$ is more important than $B$ with $a $ preferred to $\bar a$ and $  b$ preferred to $\bar b$. We get
$$
  a \bar b ≻_\phi \bar a  \bar b ≻_\phi a b  ≻_\phi \bar a b
  \ \ \ \text{ and }\ \ \ 
  a b ≻_\psi a \bar b ≻_\psi \bar a b  ≻_\psi \bar a \bar b
$$
 
Now, $ \bar  a  \bar b ≻_\phi  a b $ and $ a b ≻_\psi a \bar b$, but $ \bar  a  \bar b \not ≻_\phi a \bar b$ and $ \bar  a  \bar b \not ≻_\psi a \bar b$.
\jerome{j'ai changé la fin de l'exemple car je ne comprenais pas l'autre}
Note that $ φ $ and $ ψ  $ can be represented with additive utilities and with CP-nets.
\end{example}

This shows that none of the languages studied in this paper is complete for disjunction.

%
%
%
%

\subsection{Variable elimination}

We next consider transformations where the information is \textit{projected} onto a subset of the initial variables of interest. This is also called \textit{variable elimination}. In an interactive setting, like product configuration, it enables the user to focus on her preferences over a subset of the variables, which may be less daunting than considering the preferences over the entire set of variables. Variable elimination is a well-known technique in propositional logic, as well as in many graphical models like Bayesian networks or constraint satisfaction problems (weighted or not), where it is a component of some efficient query answering algorithms, that can be used for instance for GAIs.  (See e.g. \cite{CooperdeGivrySchiex:stacs20} for a recent unified description and overview of algorithmic aspects of graphical models.) 

Variable elimination has not been studied much in the context of preferences in general. In a pioneering work,  \cite{BesnardLangMarquis:mpref05,BesnardLangMarquis:ecai06} distinguish several ways to define the projection of a preference relation onto a subset of variables. Given a preorder $ ≽ $ and a set of variable $ U ⊆ \cal X  $, let $ V = \cal X \setminus U $, they first consider two relations defined on $ \dom V $:

\begin{description}
\item[Lower projection] $ v ≽^{↓V}_{\text{low}} v' $ if and only if $ uv ≽ uv' $ for every $ u ∈ \dom U $;
\item[Upper projection] $ v ≽^{↓V}_{\text{up}} v' $ if and only if $ uv ≽ uv' $ for some $ u ∈ \dom U $.
\end{description}

It is easy to see that  $≽^{↓V}_{\text{low}}$ and $≽^{↓V}_{\text{up}}$  are respectively the conjunction and  the disjunction  of the relations obtained by conditioning the original relation by every combination of value for $\cal X \setminus U $.

Let us also consider what \cite{BesnardLangMarquis:ecai06} call the \emph{optimistic} projections of $ ≽ $ on $ \dom V $:

\begin{description}
\item[Weak optimistic projection] $  v ≽^{↓V}_{\text{w.opt.}} v' $ if and only if for every $ u' ∈ \dom U $, there is $ u ∈ \dom U $, such that $ uv ≽ u'v' $.
\item[Strong optimistic projection]$  v ≽^{↓V}_{\text{s.opt.}} v' $ if and only if there is $ u ∈ \dom U $, such that for every $ u' ∈ \dom U $, $ uv ≽ u'v' $.
\end{description}

 
\ihel{j'ai remplacé le "contains" par le "extends" de plus haut et explicité la siginifcation - j espere ne pas m etre trompée}    

\cite{BesnardLangMarquis:mpref05} prove that $ ≽^{↓V}_{\text{w.opt.}} $  extends $ ≽^{↓V}_{\text{low}} $ and $ ≽^{↓V}_{\text{s.opt.}} $ (if $ o ≽^{↓V}_{\text{low}} o' $ (resp.  $ o ≽^{↓V}_{\text{s.opt.}} o'$), then $ o ≽^{↓V}_{\text{w.opt.}} o'$), and that the weak and strong optimistic projections are identical when $ ≽ $ is a weak order.
\footnote{\cite{BesnardLangMarquis:mpref05} also define \emph{pessimistic} counterparts of the optimistic projections: $  v ≽^{↓V}_{\text{w.pess.}} v' $ if and only if for every $ u ∈ \dom U $, there is $ u' ∈ \dom U $, such that $ uv ≽ u'v' $; and  $  v ≽^{↓V}_{\text{s.pess.}} v' $ if and only if there is $ u' ∈ \dom U $, such that for every $ u ∈ \dom U $, $ uv ≽ u'v' $. We do not consider them here because their significance, from a decision making point of view, is not clear.}

\begin{propositionrep}\label{prop:var-elim}
Given a preorder $ ≽ $ over $ \dom{\cal X} $, given $ V ⊆ \cal X $, let $ U ⊆ \cal X ∖ V $.
If $ v,v' ∈ \dom V $ and $ v ≽^{↓V}_{\text{w.opt.}} v' $, then there is some $ u ∈ \dom U $ such that for no $ u' ∈ \dom U $ it holds that $ u'v' ≻ uv $. 
\end{propositionrep}

\begin{appendixproof}
Assume that $ v ≽^{↓V}_{\text{w.opt.}} v' $, and let $ u ∈ \dom U $ be such that for no other $ u₁ ∈ \dom U $ it is the case that $ u₁v ≻ uv $: such a $ u $ must exist because $ \dom U $ is finite; if there is some completion $ u' $ of $ v' $ such that $ u'v' ≻ uv $, then, since $ v ≽^{↓V}_{\text{w.opt.}} v' $, there must be some $ u₁ ∈ \dom U $ such that $ u₁v ≽ u'v' $, but then $ u₁v  ≻ uv $, which is a contradiction.
\end{appendixproof}

The proposition above indicates that if $ v ≽^{↓V}_{\text{w.opt.}} v' $, then a decision maker may safely focus on $ v $, without risking missing a strictly better full alternative that would extend $ v '$. Note that this holds too
if $ v ≽^{↓V}_{\text{s.opt.}} v' $ or if $ v ≽^{↓V}_{\text{low}} v' $, since both imply $ v ≽^{↓V}_{\text{w.opt.}} v' $. From a preference representation point of view, the weak optimistic projection seems more interesting, as it is the one that keeps the most information -- it contains the other two.

We define the following transformations, for any projection $ π ∈ \{\text{low}, \text{up},\text{s.opt},\text{w.opt}\} $:

\defquery{$π$-projection} Given some language $ \cal L $, some formula $ φ ∈ \cal L $, some subset of variables $ V $, return a formula $ ψ ∈ \cal L $ such that $ ≽_ψ = (≽_φ)^{↓V}_π $.!



The next example shows that the languages §GAI§\ and §CP-net§, $ §CPnet§\acycl $ and $ §CPnet§\acyclpoly $
are not closed under lower projection.

\begin{example}
Consider the  preference relation $ab ≻ a \bar b ≻ \bar a \bar b ≻ \bar a b$.
This relation can be represented with a
utility function, and by an acyclic, polytree CP-net where $ A $ has no parent and is the only parent of $ B $.

The elimination of $A$ by lower projection will lead to the preorder in which the only two alternatives $b$ and $\bar b$ are incomparable due to $a b ≻ a \bar b$ and  $\bar a b ≺ \bar a \bar b$. Thus the lower projection cannot be expressed by a GAI, nor with a CP-net since a CP-net over $ \{B\} $ has only one node labelled with $ B $ and the associated table must order the pair of values $ \{b,\bar b\} $;
nor with a complete LP-tree.
\end{example}


We will next see that the 1-CP families are not closed under lower projection either.

\begin{example}
The idea is to modify the construction in Example~\ref{ex:1cpconjunction} by adding an additional variable $C$ to simulate conjunction. So we consider the §1-CP⋫∧§ formula $$\varphi = \{cb:\bar a > a, ca: b > \bar b, c\bar b: a>\bar a, \bar ca:\bar b > b, \bar cb: a > \bar a, \bar c\bar a: b>\bar b\}.$$ In the resulting preorder we have 
\[c \bar ab \succeq_\varphi cab \succeq_\varphi ca\bar b\succeq_\varphi c \bar a\bar b,
\bar ca\bar b \succeq_\varphi \bar cab \succeq_\varphi \bar c\bar a b\succeq_\varphi\bar c\bar a\bar b,\]
and all other pairs of alternatives are incomparable.

When eliminating $C$ by lower projection, we get $ab\succeq^{↓\{A,B\}}_{\varphi, low} \bar a \bar b$ and $ab$ is incomparable to the other alternatives. However, as we have seen before, this cannot be expressed by a §1-CP§, since we cannot go from $ab$ to $\bar a\bar b$ in a §1-CP§\ without any intermediate flips.
\end{example}

Most of the languages considered in this paper are not complete for the upper projection, because this projection may lead to a non transitive relation   (since the disjunction of two relations is not necessarity transitive), except $ §LPT§^{§lin§} $ and $ §GAI§₁ $. 
\footnote{\cite{BesnardLangMarquis:mpref05,BesnardLangMarquis:ecai06} in fact define the \emph{upper projection} to be the transitive closure of the relation that we denote by $ ≽^{↓V}_U $ above. However this means completing a relation when there is not necessarily a justification to do so. }

\begin{example}  We simply construct an LP tree on $\{A,B,C\}$ that has $C$ as the attribute in the root with two children; for the left child of $C$, the LPT tree for $\varphi$ from Example~\ref{ex:disjunction} is used, for the right child LPT for the relation $\psi$ is used. Then, when we apply upper projection on $C$, we get the disjunction of the orders $\varphi$ and $\psi$ which, as argued before, cannot be expressed as an LPT.\istefan{I hope I understood this correctly. Can you check?}
 
 As previously, the same counter example holds when considering GAI nets.\istefan{but how do we express the order in which the C is contained? Am I missing something simple?}

 \end{example}

 
Interestingly, linear 1-LP trees and additive utilities ($§GAI§₁$) avoid the problems in these two counterexamples.

\begin{propositionrep}
All four projections defined above are equivalent for the §1-GAI§\ language and the language that contains complete LP-trees of $ §1-LPT§^{§lin§} $, and can be computed in polynomial time.
\end{propositionrep}

\begin{appendixproof}
Let $ φ ∈ §1-GAI§ $, let $ X $ be any attribute, φ is dfined by a function of the form $ g(o) = g_X(o[X]) + \sum_{Y ≠ X} g_Y(o[Y]) $. Let $ ≽ $ denote the associated weak order over the set of alternatives. Let ψ be defined by $ h(o) =  \sum_{Y ≠ X} g_Y(o[Y]) $. We show that $ o ≽_{\text{low}} o'$ iff $ o ≽_{\text{up}} o'$ iff $ o ≽_{\text{w.opt.}} o'$ iff $ o ≽_{\text{s.opt}} o'$ iff $ o ≽_ψ o'$:
\begin{equation*}\begin{split}
o ≽_ψ o' & ⇔  h(o) ≥ h(o') \\
&  ⇔  ( \forall x ∈ \dom X : g_X(x) + h(o) ≥ g_X(x) + h(o')  ) ⇔  o ≽_{\text{low}} o' \\
&  ⇔  ( \exists x ∈ \dom X : g_X(x) + h(o) ≥ g_X(x) + h(o')  ) ⇔  o ≽_{\text{up}} o' \\
&  ⇔  ( \forall x' ∈ \dom X \exists x ∈ \dom X : g_X(x) + h(o) ≥ g_X(x') + h(o')  )
  \hfill \text{ (take } x=x' \textbf{)} \\
& \mskip 60mu ⇔  o ≽_{\text{w.opt.}} o' \\
&  ⇔  ( \forall x' ∈ \dom X : g_X(\mathop{§argmax§}\limits_{x' ∈ \dom X} g_X(x')) + h(o) ≥ g_X(x') + h(o')  ) \\
& \mskip 60mu ⇔  o ≽_{\text{s.opt.}} o'
\end{split}\end{equation*}

Suppose now that φ is a complete, linear LP-tree in §1-LPT§\ over $ \cal X $, and let $ X ∈ \cal X $. Let ψ be the LP-tree defined by removing node $ X $, redirecting the parent of $ X $ to the unique child of $ X $ when $ X $ is an internal node of φ. Consider alternatives $ o, o' ∈  \dom{\cal X \setminus X} $, let $ x ∈ \dom X $. Let $ Y $ be the attribute that decides the pair $ \{o,o'\} $ in ψ then $ o ≽_ψ o' $ iff $ o[Y] > o'[Y] $ in $ §CPT§(Y) $. Suppose first that $ Y $ is an ancestor of $ X $ in φ, then $ \{ox,o'x'\} $ is decided at $ Y $ in $ φ $ forall $ x,x ∈ \dom X $, thus $ o[Y] > o'[Y] ⇒ ( \forall x,x' ∈ \dom X : ox ≽_φ o'x' ) ⇒ o ≽_π o' $ for all four projections; and $ o ≽_π o' ⇒ o ≽_ψ o' $ for any of the four projections. Suppose now that $ X $ is an ancestor of $ Y $ in φ, then for all $ x, x' ∈ \dom X $ 1) the pair $ \{ox,o'x\} $ is decided at $ Y $, whereas 2) pair $ \{ox,o'x'\} $ with $ x ≠ x' $ is decided at $ X $. From 1) it follows that $ o ≽_ψ o' ⇔ o ≽_{\text{low}} o' ⇔ o ≽_{\text{up}} o' $; moreover, let $ x_0 $ be the optimal value for $ \dom X $ in $ §CPT§(X) $ in φ (which exists and is unique because φ is a complete LP-tree, thus the linear order over $ \dom X $ in $ §CPT§(X) $ is a linear order), then $ ox_0 ≽_φ o'x' $ for all $ x' ∈ \dom X $, thus $ o ≽_{\text{s.opt.}} o' $;   \cite{BesnardLangMarquis:ecai06} mention that $ o ≽_{\text{s.opt}} ⇒ o ≽_{\text{w.opt}} o' $; lastly, if $ o ≽_{\text{w.opt}} o' $ then there exists $ x $ such that $ ox ≽_φ o'x_0 $, and $ ox_0 ≽_φ ox $, thus $ ox_0 ≽_φ o'x_0 $, thus $ o[Y] > o'[Y] $.
\end{appendixproof}

We next show that if both  the conditioning and the weak optimistic projection can be done in polynomial time on a language, then an undominated alternative $o$ can be obtained in polynomial time, as well.

\begin{propositionrep}\label{ProjIsHard}
If conditioning can be done in polynomial time for language $ \cal L$ but the extraction of an undominated alternative is $\NP$-hard,
then the strong optimistic projections cannot be computed in polynomial time for $ \cal L $ (unless $\P = \NP$).
\end{propositionrep}

\jerome{Il faut vérifier si $§GAI§_k$ est stable pour ces projections.}
\hel{Elle ne l est pas  pour GAIk ; mais on s'en fiche : dans la demo, on projette sur une seule variable à la fois, donc sur du GAI1 garanti}
\hel{ j'ai remplacé "optimal" par "un dominated", il me semble que ca passe}
\begin{appendixproof}
We assume that for any preorder expressed in the language $ \cal L$, any strong optimistic projection leads to a preorder that again can be expressed in $ \cal L$. If this is not true, then the statement of the theorem is trivially true, even without the assumption $\P \ne \NP$.

We give an algorithm that, given a preorder $\succeq$ encoded in $ \cal L $, computes an undominated alternative $o$ in polynomial time, assuming polynomial time algorithms for conditioning and computation of strong optimistic projection. The algorithm considers the attributes of $\succeq$ in sequence, say from $V_1$ to $V_n$. The value $v_1$ of $V_1$ is obtained by projecting $\succeq$ onto $V_1$ - then an undominated value $v_1$ is chosen for $V_1$; indeed, $v_1 ≽^{↓\{V1\}}_{\text{s.opt.}} v_1'$ means that there exist an assignment $v$ of $\{V_2, \dots, V_n\}$ such that $v_1. v   ≽ v_1'. v'$ for all $v'$ -- $V_1 = v_1$ in one of the non dominated solutions. Then the original formula is conditioned: value $v_1$ is assigned to $V_1$ and the procedure is repeated for the next variable - and this until all the variables have been assigned.  So, if a language offers the conditioning transformation in polytime but not the undominated query,  there cannot be any polynomial algorithm for performing the strong optimistic projection within this language (unless $P= NP$).
\end{appendixproof}

A direct corollary is that §GAI§\ and $§GAI§_k$ ($k > 1$) fail to provide strong  and weak optimistic projections in polynomial time (unless $\P = \NP$) - since (i) the extraction of an undominated alternative is $\NP$-hard for these languages and (ii) they support  conditioning in polytime.

\begin{propositionrep}
The strong (resp. weak) projection cannot be computed in polytime for §GAI§\ and $§GAI§_k$ ($k > 1$) (unless $\P = \NP$).
\end{propositionrep}

\begin{appendixproof}
The extraction of an undominated alternative is $\NP$-hard  for  §GAI§\ and $§GAI§_k$, $k > 1$   (Proposition \ref{prop:cutextractiongai})  while    conditioning can be done in polytime for these languages. From Proposition \ref{ProjIsHard} we deduce the  strong optimistic projection cannot be computed in polytime unless $P = NP$. Because $GAI$'s encode complete and transitive relations, the strong and weak optimistic projections are identical \cite{BesnardLangMarquis:mpref05} - hence the weak optimistic projection cannot be computed in polytime unless $P = NP$.

\end{appendixproof}

The next example shows that the weak and strong optimistic projections of the preorder induced by an acyclic CP-net cannot always be represented 
in §1-CP§.

\begin{example}
Consider a CP-net $ N $ over three binary attributes $ A $, $ B $ and $ C $, with respective domains $ \{a,\bar a\} $, $ \{b,\bar b\} $, $ \{c,\bar c\} $:
\begin{center}\begin{tikzpicture}[x=6em,y=1em]
\node[var] (A) at (0,0) {$A$}; \node[anchor=north] at (A.south) {$a ≥ \bar a$};
\node[var] (B) at (1,0) {$B$}; \node[anchor=north] at (B.south) {$\array c a : b≥ \bar b\\\bar a : \bar b≥ b\endarray$};
\node[var] (C) at (2,0) {$C$}; \node[anchor=north] at (C.south) {$\array c ab,\bar a\bar b : c≥ \bar c\\\bar ab,a\bar b : \bar c ≥ c\endarray$};
\draw[pref] (A)--(B);
\draw[pref] (B)--(C);
\draw[pref] (A) to[out=30,in=150] (C);
\end{tikzpicture}\end{center}
Let $ ≽ $ denote the linear order represented by $ N $:
$ abc ≽ ab\bar c ≽ a\bar b\bar c ≽ a\bar bc ≽ \bar a\bar bc ≽ \bar a\bar b\bar c ≽ \bar ab\bar c ≽ \bar abc $, and it can be checked that $ ≽_{\text{w.opt.}}^{↓AC} $ is the relation that corresponds to the set of CP-statements $ \{a ≥\bar a, c ≥ \bar c, a\bar c ≥ \bar ac \} $, which is not included in §1-CP§: they correspond to a CP-net where $ A $ and $ C $ are preferentially independent but with the additional trade-off $ a\bar c ≥ \bar ac $. Since $ ≽ $ is complete, $ ≽_{\text{s.opt.}}^{↓AC} $ is the same as $ ≽_{\text{w.opt.}}^{↓AC} $ \cite{BesnardLangMarquis:ecai06}.
\end{example}

\section{Conclusion}

The literature on languages on CP statements has long focused on statements with unary swaps. Several examples in Section~\ref{sec:expressiveness} show that this strongly degrades expressiveness.
\ifklptcompat
We have introduced a new parameterized family of languages, $ §CP§\acycllexk $, which permits to balance expressiveness against query complexity: the lower $k $ is, the less expressive the language is, but the faster answering most queries will be.
\fi
Table~\ref{table:queries} shows that comparison queries seem to resist tractability for CP-statements,
\ifklptcompat
even for $ §CP§\acycllexk $,
\fi
but the \pb{top-$p$} query may be sufficient in many applications. The practical interest of CP-nets also lies in the fact that with this language, finding an optimal (undominated) alternative is easy~\cite{Boutilieretal:jair04}.

Contrastingly, with GAIs, it is easy to compare alternatives, but computing an undominated alternative is only tractable in the very restrictive case of additive utilities ($§GAI§₁$).

Tractability of the \pb{equivalence} query relies on the existence of canonical form: it is the case when the language enforces a structure like a dependency graph or a tree, and when the conditions of the statements are restricted to some propositional language with a canonical form.

As for transformations, the languages of (generalized) additive utilities and LP trees seem to offer better prospects, as in both cases conditioning is tractable -- whereas conditioning a formula of the most studied sublanguages of §CP§\ does not always result in a formula in the same language. Note however that for projections, tractability necessitates very strong restrictions (it only holds for $ §GAI§₁ $ and $ §1-LPT§^{§lin§} $.

An important direction for future work is to study the properties of the various languages studied here with respect to machine learning: in some context, preferences can be learnt, either through some interaction with the current user of a system, or from data gathered during past interactions. The complexity of this learning phase can influence the choice of preference model, depending on the type of interaction and on the amount of data available, and also on the computational complexity of the learning algorithms. Preliminary results about the complexity of learning CP-nets, GAIs, LP-trees can be found in e.g.~\cite{Boothetal:ecai10,KoricheZanuttini:aij10,Chevaleyreetal:plbook11,Bigotetal:pl12,AlanaziMouhoubZilles:ijcai16,AllenSilerGoldsmith:flairs17,AlanaziMouhoubZilles:art-int20,FargierGimenezMenginNguyen:mpref22}.

\begin{acks}
We thank anonymous referees for their valuable comments. This work has benefited from the AI Interdisciplinary Institute ANITI. ANITI is funded by the French "Investing for the Future -- PIA3" program under grant agreement ANR-19-PI3A-0004. This work has also been supported by the PING/ACK project of the French National Agency for Research, grant agreement ANR-18-CE40-0011.
\end{acks}

\ifAMAI
\paragraph{Conflict of Interest:} The authors declare that they have no conflict of interest.
\fi


\newcommand{\etalchar}[1]{$^{#1}$}

\end{document}